\documentclass[10pt,journal,compsoc]{IEEEtran}
\IEEEoverridecommandlockouts
\usepackage{cite}
\usepackage{amsmath,amsthm,amssymb,amsfonts}
\usepackage[linesnumbered,ruled,vlined]{algorithm2e}
\usepackage{threeparttable}
\usepackage{tabularx}
\usepackage{graphicx}
\usepackage{subcaption}
\usepackage{textcomp}
\usepackage[usenames,dvipsnames,table,xcdraw]{xcolor}
\usepackage{tkz-graph}
\usepackage{makecell}
\usepackage{bbm}
\usepackage{todonotes}
\usepackage[hidelinks]{hyperref}
\usepackage{stfloats}
\usepackage{float}
\usepackage{xspace}
\usepackage{comment}
\usepackage[normalem]{ulem}

\usetikzlibrary{arrows,decorations,backgrounds,shapes,patterns,shadows,positioning,decorations.pathmorphing,matrix,calc,automata}

\def\BibTeX{{\rm B\kern-.05em{\sc i\kern-.025em b}\kern-.08em
    T\kern-.1667em\lower.7ex\hbox{E}\kern-.125emX}}
\DeclareMathOperator*{\minimize}{minimize}

\SetKw{Or}{\textbf or}
\SetKw{And}{\textbf and}

\newtheorem{theorem}{Theorem}[section]
\newtheorem{corollary}{Corollary}[theorem]
\newtheorem{lemma}[theorem]{Lemma}
\newtheorem{example}{Example}
\theoremstyle{definition}
\newtheorem{definition}{Definition}[section]

\SetCommentSty{mycommfont}

\newcommand{\mytabcolsepsmall}{3pt}

\newcommand{\cwcet}{Blue}
\newcommand{\ceft}{Green}
\newcommand{\clst}{Red}
\newcommand{\clw}{Cyan}
\newcommand{\ciw}{Orange}
\newcommand{\cow}{Purple}

\newcommand{\matr}[1]{\mathbf{#1}}
\newcommand{\mathset}[1]{\mathcal{#1}}

\newif\iffinal
\finaltrue 

\iffinal
    \newcommand{\remove}[1]{}
    \newcommand{\removeeq}[2]{}
    \newcommand{\revise}[1]{#1}
    \newcommand{\secondr}[1]{#1}
\else
    \newcommand{\remove}[1]{}
    \newcommand{\removeeq}[2]{}
    \newcommand{\revise}[1]{#1}
    \newcommand{\secondr}[1]{{\color{blue}#1}}
\fi

\definecolorseries{test}{rgb}{grad}[rgb]{.95,.55,.55}{11,11,17}
\resetcolorseries[10]{test}
\newcommand{\addtodoeditor}[1]{%
    \colorlet{#1}{test!!+!50}
    \expandafter\newcommand\csname#1\endcsname [1]{%
        \todo[color=#1,size=\tiny]{\sffamily\textbf{\uppercase{#1}:}
    ##1}\xspace%
    }
    \expandafter\newcommand\csname#1i\endcsname [1]{%
        \todo[inline, color=#1]{\sffamily\textbf{\uppercase{#1}:} ##1}\xspace%
    }
}

\addtodoeditor{bs}
\addtodoeditor{mt}
\addtodoeditor{ab}
\addtodoeditor{dr}
\addtodoeditor{db}
\addtodoeditor{zq}

\newcommand{\eg}{{\it e.g.}\xspace}
\newcommand{\ie}{{\it i.e.}\xspace}

\let\oldnl\nl
\newcommand{\nonl}{\renewcommand{\nl}{\let\nl\oldnl}}

\begin{document}

\title{Edge Generation Scheduling for DAG Tasks\\ Using Deep Reinforcement Learning}

\author{
Binqi~Sun,
Mirco~Theile,
Ziyuan~Qin,
Daniele~Bernardini,
Debayan~Roy,
Andrea~Bastoni,
and~Marco~Caccamo%
\IEEEcompsocitemizethanks{
\IEEEcompsocthanksitem Binqi~Sun, Mirco~Theile, Ziyuan~Qin, Daniele~Bernardini, Andrea~Bastoni, and~Marco~Caccamo are with TUM School of Engineering and Design, Technical University of Munich, 85748 Munich, Germany. E-mail: \{binqi.sun, mirco.theile, ziyuan.qin, daniele.bernardini, andrea.bastoni, mcaccamo\}@tum.de.
\IEEEcompsocthanksitem Mirco~Theile is also with the Department of Electrical Engineering and Computer Sciences, University of California Berkeley, Berkeley, CA 94720, USA.
\IEEEcompsocthanksitem Debayan~Roy was formerly with TUM School of Engineering and Design, Technical University of Munich, 85748 Munich, Germany. E-mail: debayan.roy.tum@gmail.com
}%
\thanks{Marco Caccamo was supported by an Alexander von Humboldt Professorship endowed by the German Federal Ministry of Education and Research.}%
}

\IEEEtitleabstractindextext{%
\begin{abstract}
Directed acyclic graph (DAG) tasks are currently adopted in the real-time domain to model complex applications from the automotive, avionics, and industrial domains that implement their functionalities through chains of intercommunicating tasks.
This paper studies the problem of scheduling real-time DAG tasks by
presenting a novel schedulability test based on the concept of \emph{trivial schedulability}.
Using this schedulability test, we propose a new DAG scheduling framework (\emph{edge generation scheduling---EGS}) that attempts to minimize the DAG width by iteratively generating edges while guaranteeing the deadline constraint.
We study how to efficiently solve the problem of generating edges by developing a deep reinforcement learning algorithm combined with a graph representation neural network to learn an efficient edge generation policy for EGS.
We evaluate the effectiveness of the proposed algorithm by comparing it with state-of-the-art DAG scheduling heuristics and an optimal mixed-integer linear programming baseline. Experimental results show that the proposed algorithm outperforms the state-of-the-art by requiring fewer processors to schedule the same DAG tasks.
The code is available at \url{https://github.com/binqi-sun/egs}.
\end{abstract}

}

\IEEEoverridecommandlockouts
\IEEEpubid{\makebox{© 2024 IEEE. Personal use is permitted, but republication/redistribution requires IEEE permission. Citation information: DOI 10.1109/TC.2024.3350243 \hfill}}

\maketitle

\IEEEpubidadjcol

\IEEEraisesectionheading{\section{Introduction}
\label{sec:introduction}}

Current real-time applications in the automotive, avionics, and industrial domains realize their functionalities through complex chains of intercommunicating tasks. For example, \cite{andreozzi_et_al:LIPIcs.ECRTS.2022.1, waters-challenge-2019}
present recent driving assistance and autonomous driving applications where data is processed through multiple periodically-activated steps, from sensor data acquisition (\eg, Lidar and cameras) to actuators (\eg, brakes and steering wheel).
Such applications---including their execution and precedence-constraint requirements---are modeled using directed acyclic graph (DAG) tasks with different periods that can be reduced to a \emph{single DAG} using techniques such as~\cite{verucchi2020latency}.

Reasoning on real-time properties of DAG tasks has proved challenging. Testing the schedulability of DAG tasks is NP-hard in the strong sense~\cite{ullman-np-complete-sched-prob}, and many works have focused on devising heuristics for sufficient schedulability tests (see Section~\ref{sec:related_work} for relevant related works and refer to~\secondr{\cite{micaela-thesis,Li2022,verucchi2023survey}} for a comprehensive survey), while exact results \secondr{(\eg,~\cite{baruah-sched-dag-assignment,chang-exact-dag-wcrt,ahmed2022exact})} can only be obtained for simple DAG tasks.

This work focuses on the setup where a single periodic non-preemptive DAG task is executed on a multicore platform with identical processors. Drawing from graph theory, we develop a novel schedulability test based on the key observation that a DAG whose \emph{width} is not greater than the number of available processors and whose \emph{length} is less than or equal to the deadline of the DAG is schedulable.
We classify such a DAG as a \emph{trivially schedulable} DAG and show that any DAG is schedulable if and only if it can be converted into a trivially schedulable DAG by adding edges. In addition, we show that a trivially schedulable DAG task can be dispatched via \emph{global} and \emph{partitioned} strategies. While global dispatching strategies usually require prioritized queues for ready jobs, we show that prioritization is not needed when dispatching a trivially schedulable DAG task because a ready job is guaranteed to have an idle processor available for execution. For partitioned dispatching strategies, the paths covering a DAG can simply be assigned to processors in the order of the precedence constraints.

To test whether a DAG task is schedulable, we then focus on the problem of adding appropriate edges to convert it into a trivially schedulable DAG task without violating its original constraints. To this end, we propose the \emph{Edge Generation Scheduling} (EGS) framework that attempts to make a DAG task trivially schedulable by iteratively adding appropriately chosen edges. If EGS succeeds in reducing the width to the number of processors while maintaining the length less than or equal to the deadline, the original DAG task is guaranteed to be schedulable. 

The EGS framework shifts the complexity of solving the DAG scheduling problem to the problem of selecting \emph{the best} edges to add to a DAG to make it trivially schedulable. We exploit topological and temporal graph properties to limit the search space for the edges to add and propose a deep reinforcement learning (DRL) approach to learn an edge generation policy. In particular, we use the DRL algorithm Proximal Policy Optimization (PPO) \cite{schulman2017proximal} and the graph representation neural network architecture Graphormer~\cite{ying2021transformers} that is well suited for solving this class of problems. 

Combining the proposed EGS framework and the edge generation policy learned by the developed DRL, we derive a concrete DAG scheduling algorithm called EGS-PPO and evaluate it against state-of-the-art DAG scheduling heuristics. Our results show that EGS-PPO consistently outperforms the other approaches by requiring fewer processors to schedule the same DAGs. Additionally, EGS with a random edge generation policy can achieve results similar to the state-of-the-art, highlighting the significance of the EGS framework.
We also compare against an optimal mixed-integer linear programming (MILP) baseline for small DAG tasks. EGS-PPO outperforms the other approaches, achieving three to five times smaller optimality gaps.

In summary, in this paper, we:  
\begin{enumerate}
    \item Present a new schedulability test (trivial schedulability) for DAG tasks based on observations from the graph domain;
    \item Propose a novel DAG scheduling framework (EGS) that minimizes processor usage by iteratively generating edges; 
    \item Formulate the edge generation problem as a Markov decision process (MDP) and develop a deep reinforcement learning (DRL) agent to learn an effective edge generation policy for EGS;
    \item Evaluate the effectiveness of the proposed EGS framework and DRL algorithm by comparing with exact solutions and state-of-the-art DAG scheduling algorithms through extensive experiments on synthetic DAG tasks.
\end{enumerate}

The remainder of the paper is organized as follows. Section~\ref{sec:related_work} reviews the literature on DAG scheduling, and Section~\ref{sec:preliminaries} describes the system model and introduces the employed concepts of graph theory. Section~\ref{sec:sched_test} and Section~\ref{sec:egs} present our schedulability test and the EGS scheduling framework. A DRL algorithm is developed in Section~\ref{sec:drl} to learn an efficient edge generation policy for EGS. Section~\ref{sec:evaluation} discusses our experimental evaluation, and Section~\ref{sec:conclusion} presents future research directions and conclusions.

\section{Related Work}
\label{sec:related_work}

\subsection{Real-time DAG scheduling}
The periodic computation in many cyber-physical systems~(CPS) domains, such as automotive, avionics, and manufacturing, is often modeled as a DAG task~\cite{waters-challenge-2019,Minaeva2021}. Many applications in these domains are part of safety-critical control loops (\eg, brake, speed, and steering control) and hence, have stringent timing requirements (\ie, they are required to meet their deadlines~\cite{AUTOSARTiming,ARINC653}).
\secondr{These requirements led to a body of work performing timing and scheduling analyses of a variety of DAG-based software models,
ranging from a single DAG modeling one task~\cite{Baruah12,Graham69,he2019intra,zhao2020dag,he2021response,he2022bounding}, DAGs for multiple tasks with different periods~\cite{Bonifaci13,Baruah14,li2014federated,Pathan18,yadlapalli2021lag,zhao2022dag}, and more recently, to conditional DAGs~\cite{Baruah15,melani2015response,ueter2021response}, heterogeneous DAGs~\cite{Yang16,Chang20,Zahaf21,reghenzani2021multi}, and DAGs with mutually exclusive vertices~\cite{bi2022response}}. 
\revise{Given the techniques proposed in~\cite{verucchi2020latency} to reduce a DAG task set to a single DAG,} this work focuses on single DAGs for establishing a new scheduling strategy. \revise{We note that our method can be trivially extended to the scheduling of multiple DAGs with federated scheduling architectures following the approach in ~\cite{li2014federated}.} 
\secondr{For brevity, the following focuses on real-time DAG scheduling closely related to this work. We refer the readers to~\cite{micaela-thesis,Li2022,verucchi2023survey} for a comprehensive survey.}

The real-time DAG scheduling literature has mainly performed analysis to (i)~derive schedulability tests, (ii)~bound the response times, and (iii)~put forward scheduling strategies to improve schedulability. 
Baruah et al.~\cite{Baruah12} first proposed a schedulability test for a single DAG task with constrained deadlines and the global earliest deadline first (EDF) scheduling policy. The test is mainly based on the task's deadline and period, the length of the longest task chain, and the volume (\ie, the sum of the WCETs) of the DAG. Later, Bonifaci et al.~\cite{Bonifaci13} extended the test to a global deadline monotonic (DM) scheduling policy, arbitrary deadlines, and a set of DAG tasks. 
Furthermore, Baruah et al.~\cite{Baruah14} improved the schedulability test for constrained deadlines by exploiting the concept of work functions.

One of the earliest works in the DAG response time analysis provided a bound---popularly known as Graham's bound---for the response time of a task based on the longest path in and the volume of the DAG~\cite{Graham69}, \revise{which is valid for any work-conserving scheduling policy on homogeneous multicore platforms. Recently, He et al.~\cite{he2022bounding} demonstrated the pessimism in Graham's bound and proposed a tighter bound considering multiple long paths instead of the longest one.} Melani et al.~\cite{melani2015response} also extended Graham's bound to systems with multiple DAG tasks by considering inter-task interference. Global earliest-deadline-first (EDF) and fixed-priority (FP) scheduling policies were studied in that work. Further, in~\cite{Pathan18}, the bounds were made tighter for two-level FP scheduling, where a DM scheme was followed at the task level, while subtasks were assigned priorities based on the topological order. In recent years, He et al.~\cite{he2019intra} proposed prioritizing subtasks in the longest paths to reduce the response time and improve schedulability, while at the task level, they still applied DM. Later, Zhao et al.~\cite{zhao2020dag} improved the priority assignment strategy at the subtask level by considering dependencies between subtasks and parallelization opportunities.
\revise{Different from the above approaches, He et al.~\cite{he2021response} relaxed the constraint that priority assignment must comply with the topological order of the DAG and proposed a new priority assignment policy, leading to smaller response time bounds.}

Note that some of the works above performed analyses for given scheduling policies~\cite{Baruah12,Bonifaci13,Baruah14,Graham69,melani2015response,he2022bounding}, while others propose techniques to determine schedule configuration (\eg, priorities) to improve the schedulability~\cite{he2019intra,zhao2020dag,he2021response}. Our work follows the latter direction, \ie, we determine a static ordering of sub-tasks that will make the task schedulable. 
In our experiments (Section~\ref{sec:evaluation}), we show that our proposed approach outperforms the most recent works~\cite{he2019intra,zhao2020dag,he2021response} in generating feasible schedules for DAG tasks.

\subsection{DRL for DAG scheduling}

Deep reinforcement learning (DRL) has been applied to various combinatorial optimization problems, including scheduling tasks. In recent years, several studies have applied DRL to DAG scheduling. 
\secondr{Mao et al.~\cite{mao2019learning} proposed a DRL-based DAG scheduler for scheduling data processing jobs in the Spark cluster. Their model takes the cluster's state information as input and learns to select the next DAG node to be executed via a graph convolution network (GCN) and a policy gradient method.
Sun et al.~\cite{sun2021deepweave} proposed a DRL approach to solve a coflow scheduling problem in distributed computing. It also uses a graph neural network (GNN) in combination with a policy gradient method. However, different from our work, it learns to schedule the edges of a DAG job representing communication stages (\ie, coflows) instead of the DAG nodes representing computation stages.

More recently, Lee et al.~\cite{Lee2021Global} proposed a DRL-based DAG task scheduler, which employs a GCN to process a complex interdependent task structure and minimize the makespan of a DAG task. The scheduler assigns priorities to each sub-task to be used in list scheduling. 
Similar to~\cite{Lee2021Global}, Joen et al.~\cite{jeon2023neural}
developed a learning-based scheduler to assign priorities for list scheduling. The difference is that they proposed a one-shot neural network encoder to sample priorities instead of using an episodic reinforcement learning approach.
In contrast to these works, our proposed method remains in the graph domain and adds edges to make the DAG task trivially schedulable. 
Since the code or implementation details of \cite{Lee2021Global} and \cite{jeon2023neural} have not been released, we cannot compare our method with
theirs in the experimental evaluation.}

\section{System Model and Preliminaries}
\label{sec:preliminaries}

\subsection{Task model}

We consider a 
DAG task running on $M$ identical processors. The DAG task is characterized by $(\mathcal{G},D\leq T)$, where $\mathcal{G}$ is a graph defining the set of sub-tasks,
$T$ denotes the task period defined as the inter-arrival time of two consecutive jobs (\textit{i.e.}, task instances),
and $D$ denotes the task deadline by which all the active sub-jobs must finish their execution.
Without loss of generality, we consider constrained deadline, which means the deadline is smaller or equal to the task period (\textit{i.e.}, $D \leq T$).
The task graph $\mathcal{G}$ is defined by $(\mathset{V},\mathset{E})$, where $\mathset{V} = (v_i)$ is a set of $n$ nodes representing $n$ sub-tasks, and $\mathset{E} = (e_{ij})$ is a set of directed edges representing the precedence constraints between the sub-tasks. Each sub-task $v_i$ is a non-preemptive sequential computing workload, and its worst-case execution time (WCET) is denoted as $C_i$. For any two nodes $v_i$ and $v_j$ connected by a directed edge $e_{ij}$, $v_j$ can start execution only if $v_i$ has finished its execution. Node $v_i$ is called a \textit{predecessor} of $v_j$, and $v_j$ is a \textit{successor} of $v_i$. The predecessors and successors of a node $v_i$ are formally defined as $pre(v_i) = \{v_j \in \mathset{V} | e_{ji} \in \mathset{E}\}$ and $suc(v_i) = \{v_j \in \mathset{V} | e_{ij} \in \mathset{E}\}$, respectively. Moreover, the nodes that are either directly or transitively predecessors (\textit{resp.}, successors) of node $v_i$ are defined as the \textit{ancestors} (\textit{resp.}, \textit{descendants}) of node $v_i$, denoted by $anc(v_i)$ (\textit{resp.}, $des(v_i)$). Furthermore, a node with no ancestor (\textit{resp.}, descendant) is referred to as the \textit{source} (\textit{resp.}, \textit{sink}) node of the DAG. Without loss of generality, we assume only one source node and one sink node exist in a DAG. A DAG with multiple source (sink) nodes can be easily supported by adding dummy nodes with zero WCET.

\begin{example}
\label{eg:dag}
Consider a DAG task $(\mathcal{G}, D)$ consisting of 7 nodes and 8 edges. The DAG $\mathcal{G}$ is shown in Fig. \ref{fig:dag_example}. The number below each node denotes its WCET. The task deadline is set as $D=8$. Take node $v_6$ as an example, the predecessors and successors of node $v_6$ are $pre(v_6) = \{v_2, v_3, v_4\}$ and $suc(v_6) = \{v_7\}$; the ancestors and descendants of node $v_6$ are $anc(v_6) = \{v_1, v_2, v_3, v_4\}$ and $des(v_6) = \{v_7\}$, respectively.
\end{example}


\begin{figure}[ht]
\centering
\begin{tikzpicture}[scale=0.75]
    [
        ->,>=stealth',shorten >=1pt,auto,node distance=2.8cm,semithick,node font=\footnotesize
    ]
    \node (1) [circle, draw, label=below:$v_1$]  at (0.0,1.5) {0};
    \node (2) [circle, draw, label=below:$v_2$]  at (2.0,3.0) {5};
    \node (3) [circle, draw, label=below:$v_3$]  at (2.0,1.5) {4};
    \node (4) [circle, draw, label=below:$v_4$]  at (2.0,0.0) {3};
    \node (5) [circle, draw, label=below:$v_5$]  at (4.0,2.25) {3};
    \node (6) [circle, draw, label=below:$v_6$]  at (4.0,0.75) {1};
    \node (7) [circle, draw, label=below:$v_7$]  at (6.0,1.5) {0};
    \path [->] (1) edge (2);
    \path [->] (1) edge (3);
    \path [->] (1) edge (4);
    \path [->] (2) edge (5);
    \path [->] (2) edge (6);
    \path [->] (3) edge (6);
    \path [->] (4) edge (6);
    \path [->] (5) edge (7);
    \path [->] (6) edge (7);

    \node[draw] at (7.0, 3.0) {$D=8$};
    
\end{tikzpicture}
\caption{Example of a DAG task.} \label{fig:dag_example}
\end{figure}
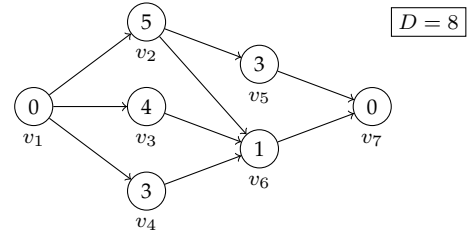

\subsection{Scheduling model}

A DAG task is considered \textit{schedulable} if all the sub-tasks can finish their execution no later than the deadline.
At runtime, a sub-job is \emph{ready} once the job is released and its predecessors have finished their execution.
We consider two different strategies for dispatching ready sub-jobs to processors: (i) a \emph{global} strategy, where sub-jobs are dynamically dispatched on the available processors and (ii) a \emph{partitioned} strategy, where the assignment of nodes to processors is predetermined offline.

The global dispatching strategy can be implemented by maintaining a prioritized queue to store the ready sub-jobs waiting for execution. Once a sub-job is ready, it goes into the waiting queue, and when a processor becomes idle, the highest priority sub-job in the queue is assigned to the processor for execution. 
Note that since the node-to-processor mapping is not fixed in the global dispatching strategy, different sub-jobs of the same node can execute on different processors. A partitioned strategy can be implemented by maintaining a list for each processor to store the nodes to be executed and their relative execution order.

\subsection{Boolean algebra in graph theory} 
\label{sec:algebra}
In graph theory, Boolean matrices are widely-used to represent graph structures. Thus, Boolean algebra can be applied. Here, we introduce some basic Boolean matrix operators and show how they are used in graph operations.

\revise{\subsubsection{Adjacency matrix and transitive closure}}
An \textit{adjacency matrix} $\matr{A} \in \mathbb{B}^{n \times n}$ is a binary square matrix used to represent the connectivity relations in a graph $\mathcal{G}=(\mathcal{V}, \mathcal{E})$, where $[\matr{A}]_{ij} = 1$ if and only if there exists an edge $e_{ij} \in \mathcal{E}$ between node $v_i$ and node $v_j$.
The \textit{transitive closure} of a DAG is defined to represent the reachability relation between the nodes. It can be represented as a binary matrix $\matr{T_c} \in \mathbb{B}^{n \times n}$, where $[\matr{T_c}]_{ij}=1$ if and only if node $v_i$ is an ancestor of node $v_j$. 

The adjacency matrix of the transitive closure can be calculated based on the adjacency matrix of the original graph by the Floyd-Warshall algorithm \cite{floyd1962algorithm}, which requires a time complexity of $\mathcal{O}(n^3)$. 

\revise{\subsubsection{Matrix multiplication}}
We introduce two matrix multiplication methods in Boolean algebra: \textit{Boolean matrix multiplication} and \textit{max-plus matrix multiplication}. \revise{The matrix multiplication is used to determine ancestors and dependents of nodes, while the max-plus matrix multiplication is used to compute execution time bounds.}

Given two Boolean matrices $\mathbf{A}, \mathbf{B} \in \mathbb{B}^{n \times n}$, the Boolean matrix multiplication $\mathbb{B}^{n\times n}\times \mathbb{B}^{n\times n}\mapsto \mathbb{B}^{n\times n}$ is defined as:
\begin{equation} \label{eq:bmm}
    [\mathbf{AB}]_{ij} = \bigvee_{k=1}^{n}{([\mathbf{A}]_{ik} \land [\mathbf{B}]_{kj})},\quad \forall i,j = 1,...,n
\end{equation}
where $[\mathbf{AB}]_{ij}$ denotes the $(i,j)$-th element in the Boolean matrix multiplication product $\mathbf{A} \mathbf{B}$; $\lor$ and $\land$ denote the \textit{or} and \textit{and} operators, respectively.

Given a Boolean matrix $\mathbf{A} \in \mathbb{B}^{n \times n}$ and a real vector $\mathbf{b} \in \mathbb{R}^{n}$, the max-plus matrix multiplication $\mathbb{B}^{n\times n}\times \mathbb{R}^{n}\mapsto \mathbb{R}^{n}$ is defined as:
\begin{equation} \label{eq:mpmm}
    [\mathbf{A} \otimes \mathbf{b}]_i = \max_{k=1,...,n}{([\mathbf{A}]_{ik} \cdot [\mathbf{b}]_k)},\quad \forall i = 1,...,n
\end{equation}
where $[\mathbf{A} \otimes \mathbf{b}]_i$ is the $i$-th element in the max-plus matrix multiplication product $\mathbf{A} \otimes \mathbf{b}$.

\subsection{Length of a DAG task}

\subsubsection{Path and DAG length}
\revise{The length of a DAG is a lower bound on the total execution time of the DAG task. It is later used as a core element of the EGS scheduler. Formally, it is defined through the length of the critical path.}
A \textit{path} $p=\{v_{p_1},...,v_{p_m}\}$ is a sequence of nodes that are connected by a sequence of edges in the same direction (\textit{i.e.}, $e_{p_kp_{k+1}} \in \mathset{E}, \forall k=1,2,...,m-1$). The length of path $p$ is defined as the sum WCET of the nodes included in the path: $\mathrm{L}(p) = \sum_{k=1}^{m}{C_{p_k}}$. A \textit{complete path} of a DAG is a path that includes the source node and sink node of the DAG. The longest complete path is defined as the \textit{critical path} $p^*$, and the length of the critical path is defined as the length of the DAG. 
More formally, we have:

\begin{definition}[DAG length]\label{def:length}
    The length of a DAG task $(\mathcal{G}, D)$ equals the length of the longest path in the DAG. $\mathrm{L}(\mathcal{G}) = \max_{p \in P}{\mathrm{L}(p)}$, where $P$ is the set of paths in $\mathcal{G}$.
\end{definition}

The length of a DAG can be computed within time complexity $\mathcal{O}(n^2)$.

\subsubsection{Node-level timing attributes}
For each node, we define four timing attributes related to the DAG length: earliest starting time (EST), earliest finishing time (EFT), latest starting time (LST), and latest finishing time (LFT). The EST means the earliest time a node can start its execution, which equals the maximum of its predecessors' EFT. Similarly, the LFT represents the latest time a node can finish its execution while meeting the deadline, \ie, the minimum of its successors' LST\revise{. These timing attributes are used to restrict the action space of the reinforcement learning agent, and are part of the node features to aid the agent's learning. They are defined through}:
\begin{equation} \label{eq:node_level_timing}
\begin{aligned}
    \mathbf{t}^\text{EFT} &= \mathbf{t}^\text{EST} + \mathbf{C}\\
    \mathbf{t}^\text{EST} &= \matr{A} \otimes \mathbf{t}^\text{EFT}\\
    \mathbf{t}^\text{LST} &= \mathbf{t}^\text{LFT} - \mathbf{C}\\
    \mathbf{t}^\text{LFT} &= \matr{A}^T \otimes \mathbf{t}^\text{LST}
\end{aligned}
\end{equation}
where $\mathbf{t}^\text{EST}$, $\mathbf{t}^\text{EFT}$, $\mathbf{t}^\text{LST}$, $\mathbf{t}^\text{LFT} \in \mathbb{R}^n$ are vectors denoting the EST, EFT, LST, and LFT of the nodes, respectively; $\mathbf{C} \in \mathbb{R}^n$ is a vector denoting the WCET of each node.

Equations (\ref{eq:node_level_timing}) can be solved by fixed-point iteration. First, we initialize $\revise{t}^\text{EST}_i = 0,\ \revise{t}^\text{LFT}_i = D,\ \forall i=1,...,n$. Then, at each iteration $k=1,...,n$, we update the values of each timing attribute according to \eqref{eq:node_level_timing} until they converge (\textit{i.e.}, no value is updated from iteration $k$ to iteration $k+1$). It can be guaranteed that the values will converge within $n$ iterations since the critical path of the DAG is composed of at most $n$ nodes \cite{verucchi2020latency}. Thus, the time complexity of computing the node-level timing attributes is $\mathcal{O}(n^3)$.

\begin{example}
The critical path of the DAG in Example \ref{eg:dag} is $\{v_1, v_2, v_5, v_7\}$, and the length of the DAG is $C_1 + C_2 + C_5 + C_7 = 8$. 
The EFT and LST of each node are shown in Fig.~\ref{fig:eg_timing}.

\begin{figure}[ht]
\centering
\begin{tikzpicture}[scale=0.85]
    [
        ->,>=stealth',shorten >=1pt,auto,node distance=2.8cm,semithick,node font=\footnotesize
    ]
    \tikzset
    {
        in place/.style=
        {
          auto=false,
          fill=white,
          inner sep=2pt,
        },
    }
    \node (1) [circle, draw, align=center, label=below:$v_1$, inner sep=1pt, red] at (0.0,1.5) 
    {
        \textcolor{\cwcet}{$0$} \\ $\textcolor{\clst}{0}$ $\textcolor{\ceft}{0}$
    };
    \node (2) [circle, draw, align=center, label=below:$v_2$, inner sep=1pt, red] at (2.0,3.0) 
    {
        \textcolor{\cwcet}{$5$} \\ $\textcolor{\clst}{0}$ $\textcolor{\ceft}{5}$
    };
    \node (3) [circle, draw, align=center, label=below:$v_3$, inner sep=1pt]  at (2.0,1.5)
    {
        \textcolor{\cwcet}{$4$} \\ $\textcolor{\clst}{3}$ $\textcolor{\ceft}{4}$
    };
    \node (4) [circle, draw, align=center, label=below:$v_4$, inner sep=1pt]  at (2.0,0.0)
    {
        \textcolor{\cwcet}{$3$} \\ $\textcolor{\clst}{4}$ $\textcolor{\ceft}{3}$
    };
    \node (5) [circle, draw, align=center, label=below:$v_5$, inner sep=1pt, red]  at (4.0,2.25)
    {
        \textcolor{\cwcet}{$3$} \\ $\textcolor{\clst}{5}$ $\textcolor{\ceft}{8}$
    };
    \node (6) [circle, draw, align=center, label=below:$v_6$, inner sep=1pt]  at (4.0,0.75)
    {
        \textcolor{\cwcet}{$1$} \\ $\textcolor{\clst}{7}$ $\textcolor{\ceft}{6}$
    };
    \node (7) [circle, draw, align=center, label=below:$v_7$, inner sep=1pt, red]  at (6.0,1.5)
    {
        \textcolor{\cwcet}{$0$} \\ $\textcolor{\clst}{8}$ $\textcolor{\ceft}{8}$
    };

    \node (legend) [circle, draw, align=center, label=below:, inner sep=1pt, font=\tiny]  at (7.0,0.0)
    {
        \textcolor{\cwcet}{WCET} \\ \textcolor{\clst}{LST} \textcolor{\ceft}{EFT}
    };
    
    \path [->, red] (1) edge (2);
    \path [->] (1) edge (3);
    \path [->] (1) edge (4);
    \path [->, red] (2) edge (5);
    \path [->] (2) edge (6);
    \path [->] (3) edge (6);
    \path [->] (4) edge (6);
    \path [->, red] (5) edge (7);
    \path [->] (6) edge (7);

    \node[draw] at (7.0, 3.0) {$D=8$};

\end{tikzpicture}
\caption{Example of DAG length and node-level timing attributes. The nodes and edges of the critical path are marked in red. The numbers inside each node represent the node's WCET, LST, and EFT with corresponding colors.}
\label{fig:eg_timing}
\end{figure}
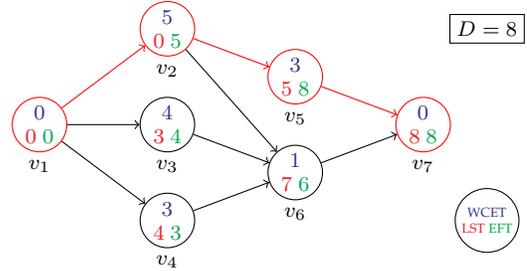%
\end{example}

\subsection{Width of a DAG task} 

\subsubsection{Antichain and DAG width}
\revise{The width of a DAG task indicates the maximum number of nodes that can be run in parallel. It can be computed using the critical antichain of the DAG.}
An antichain $q=\{v_{q_1}, ..., v_{q_m}\}$ in a DAG is a set of nodes that are pair-wise non-reachable (\textit{i.e.}, $v_i \notin anc(v_j), \forall i,j \in (q_1,...,q_m)$). The size of an antichain is defined as the number of nodes in the antichain. The maximum-size antichain is called the \textit{critical antichain}, and the size of the critical antichain is defined as the width of the DAG. By Dilworth’s theorem \cite{dilworth1950decomposition}, the width of $\mathcal{G}$ also equals the minimum number of paths needed to cover all the nodes of DAG $\mathcal{G}$.

\begin{definition}[DAG width]\label{def:width}
    The width of a DAG task $(\mathcal{G},D)$ is defined as the size of the maximum-size antichain in DAG $\mathcal{G}$.
    It is equivalent to the minimum number of paths needed to cover all the nodes in DAG $\mathcal{G}$.
\end{definition}

There have been several methods proposed for DAG width computation in the literature. The most well-known ones are minimum path cover algorithms, where the problem is reduced to either the \textit{maximum matching} in bipartite graphs with time complexity $\mathcal{O}(\sqrt{n} m^*)$ \cite{hopcroft1973n} or the \textit{minimum flows} in directed graphs with time complexity $\mathcal{O}(n m)$ \cite{bang2008digraphs}, where $m$ and $m^*$ denote the number of edges in graph $\mathcal{G}$ and the transitive closure of graph $\mathcal{G}$, respectively. 

\subsubsection{Node-level parallelism attributes}
For each node $v_i \in V$, we define three attributes related to the DAG width: lateral width (LW), in-width (IW), and out-width (OW). The LW of node $v_i$ means the maximum number of pair-wise non-reachable nodes with which node $v_i$ can run in parallel. It equals the width of the DAG derived by removing node $v_i$ and all its ancestors and descendants from DAG $\mathcal{G}$. The IW (\textit{resp.}, OW) of node $v_i$ denotes the maximum number of pair-wise non-reachable nodes among the ancestors (\textit{resp.}, descendants) of node $v_i$ and the nodes that are parallel to $v_i$. \revise{As with the node-level timing attributes, the parallelism attributes are used to restrict action space and are part of the node features intended to improve the agent's understanding of the DAG.} Similar to the LW, the IW (\textit{resp.}, OW) of node $v_i$ can be calculated as the width of the DAG after removing node $v_i$ and all its descendants (\textit{resp.}, ancestors) from DAG $\mathcal{G}$:
\begin{equation}
    \begin{aligned}
        &m^\text{LW}_i = \mathrm{W}(\mathcal{G} \setminus \mathset{V}'), \mathset{V}' = anc(v_i) \cup des(v_i) \cup \{v_i\}\\
        &m^\text{IW}_i = \mathrm{W}(\mathcal{G} \setminus \mathset{V}'),\ \mathset{V}' = des(v_i) \cup \{v_i\}\\
        &m^\text{OW}_i = \mathrm{W}(\mathcal{G} \setminus \mathset{V}'), \mathset{V}' = anc(v_i) \cup \{v_i\},\ \forall i=1,...,n\\
    \end{aligned}
\end{equation}
where $m^\text{LW}_i$, $m^\text{IW}_i$, and $m^\text{OW}_i$ denote the LW, IW, and OW of node $v_i$, respectively; $\mathcal{G} \setminus \mathset{V}'$ denotes the graph with the nodes in node set $\mathset{V}'$ and all their connected edges removed from graph $\mathcal{G}$. The time complexity of computing each node-level parallelism attribute is the same as the complexity of computing the graph width.

\begin{example}
The critical antichains of the DAG in Example \ref{eg:dag} are $\{v_2,v_3,v_4\}$ and $\{v_3,v_4,v_5\}$, and the DAG width is 3. The LW, IW and OW of each node are illustrated in Fig.~\ref{fig:eg_parallelism}.

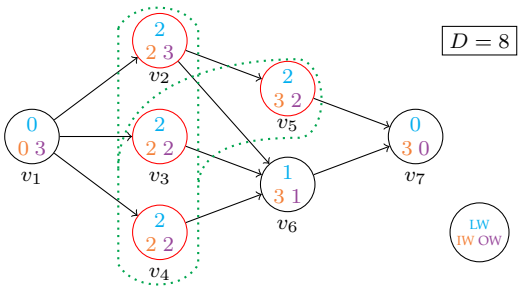
\begin{figure}[ht]
\centering
\begin{tikzpicture}[scale=0.85]
    [
        ->,>=stealth',shorten >=1pt,auto,node distance=2.8cm,semithick,node font=\footnotesize
    ]
    \tikzset
    {
        in place/.style=
        {
          auto=false,
          fill=white,
          inner sep=2pt,
        },
    }
    \node (1) [circle, draw, align=center, label={below:$v_1$}, inner sep=1pt] at (0.0,1.5) 
    {
        $\textcolor{\clw}{0}$ \\ $\textcolor{\ciw}{0}$ $\textcolor{\cow}{3}$
    };
    \node (2) [circle, draw, align=center, label={[yshift=-1.04cm]:$v_2$}, inner sep=1pt, red] at (2.0,3.0) 
    {
        $\textcolor{\clw}{2}$ \\ $\textcolor{\ciw}{2}$ $\textcolor{\cow}{3}$
    };
    \node (3) [circle, draw, align=center, label=below:$v_3$, inner sep=1pt, red]  at (2.0,1.5)
    {
        $\textcolor{\clw}{2}$ \\ $\textcolor{\ciw}{2}$ $\textcolor{\cow}{2}$
    };
    \node (4) [circle, draw, align=center, label=below:$v_4$, inner sep=1pt, red]  at (2.0,0.0)
    {
        $\textcolor{\clw}{2}$ \\ $\textcolor{\ciw}{2}$ $\textcolor{\cow}{2}$
    };
    \node (5) [circle, draw, align=center, label={[yshift=-1.04cm]:$v_5$}, inner sep=1pt, red]  at (4.0,2.25)
    {
        $\textcolor{\clw}{2}$ \\ $\textcolor{\ciw}{3}$ $\textcolor{\cow}{2}$
    };
    \node (6) [circle, draw, align=center, label=below:$v_6$, inner sep=1pt]  at (4.0,0.75)
    {
        $\textcolor{\clw}{1}$ \\ $\textcolor{\ciw}{3}$ $\textcolor{\cow}{1}$
    };
    \node (7) [circle, draw, align=center, label=below:$v_7$, inner sep=1pt]  at (6.0,1.5)
    {
        $\textcolor{\clw}{0}$ \\ $\textcolor{\ciw}{3}$ $\textcolor{\cow}{0}$
    };

    \node (legend) [circle, draw, align=center, label=below:, inner sep=1pt, font=\tiny]  at (7.0,0.0)
    {
        \textcolor{\clw}{LW} \\ \textcolor{\ciw}{IW} \textcolor{\cow}{OW}
    };
    
    \path [->] (1) edge (2);
    \path [->] (1) edge (3);
    \path [->] (1) edge (4);
    \path [->] (2) edge (5);
    \path [->] (2) edge (6);
    \path [->] (3) edge (6);
    \path [->] (4) edge (6);
    \path [->] (5) edge (7);
    \path [->] (6) edge (7);

    \draw [-, dotted, thick, Green] (1.35, 3.2)  to (1.35, -0.5) to[bend right=60]  (2.6, -0.5) to (2.6, 3.2) to[bend right=60] (1.35, 3.2);
    \draw [-, dotted, thick, Green] (1.35, 1.0) to [out=90, in=180] (4.18, 2.71) to [out=-10, in=0] (4.16, 1.48) to [out=180, in=90] (2.6, 0.7);


    \node[draw] at (7.0, 3.0) {$D=8$};
\end{tikzpicture}
\caption{Example of DAG width and node-level parallelism attributes. The nodes belonging to the critical antichains are marked red, and a dotted line surrounds each antichain. The numbers inside each node represent the node's LW, IW, and OW with corresponding colors.}
\label{fig:eg_parallelism}
\end{figure}%
\end{example}

\section{Schedulability Test}
\label{sec:sched_test}

We present an exact schedulability test for a DAG task based on the concept of \textit{trivial schedulability} defined by the length and width of the DAG: 

\begin{definition}[Trivial schedulability]\label{def:trivial_sched}
    A DAG task $(\mathcal{G},D)$ is \textit{trivially schedulable} on $M$ processors if it satisfies the following two conditions:
    \begin{enumerate}
        \item the length of $\mathcal{G}$ is no larger than the task deadline: $\mathrm{L}(\mathcal{G}) \leq D$;
        \item the width of $\mathcal{G}$ is no larger than the number of processors: $\mathrm{W}(\mathcal{G}) \leq M$.
    \end{enumerate}
\end{definition}

Based on Definition~\ref{def:trivial_sched}, we can derive several important properties of a trivially schedulable DAG task, which are later used to prove our schedulability test. 

First, we show that a trivially schedulable DAG is guaranteed to be schedulable under a global dispatching strategy by formulating the following Lemmas~\ref{lem:active_jobs}~-~\ref{lem:global_sufficient_condition}.

\begin{lemma}\label{lem:active_jobs}
    If a DAG task $(\mathcal{G},D)$ is \textit{trivially schedulable} on $M$ processors, then at most $M$ sub-jobs are active at the same time.
\end{lemma}
\begin{proof}
    We prove the lemma by contradiction. Suppose there are $M+1$ active sub-jobs at the same time. Since two sub-jobs can be active at the same time only if they do not have precedence constraints, there must be $M+1$ nodes that are pair-wise non-reachable. Thus, they constitute an antichain of size $M+1$. By Definition~\ref{def:width}, the width of $\mathcal{G}$ is thus at least $M+1$, which contradicts the width constraint $\mathrm{W}(\mathcal{G}) \leq M$ in Definition~\ref{def:trivial_sched}. 
\end{proof}

\begin{lemma}\label{lem:start_time}
    If a DAG task $(\mathcal{G},D)$ is \textit{trivially schedulable} on $M$ processors, each sub-job $v_i \in \mathcal{V}$ can start execution at its ready time using any global \revise{work-conserving} dispatching strategies.
\end{lemma}
\begin{proof}
    We prove the lemma by contradiction. Suppose a sub-job of node $v_i$ cannot start its execution at its ready time $r_i$. Since we consider a global \revise{work-conserving} dispatching strategy, all $M$ processors must be busy executing other sub-jobs at time $r_i$. 
    Thus, we know that at least $M+1$ sub-jobs (including $v_i$) are active at time $r_i$, which contradicts Lemma~\ref{lem:active_jobs}.
\end{proof}

By Lemma~\ref{lem:start_time}, we know that a trivially schedulable DAG task will not have any ready sub-jobs waiting for processors to become idle. 
Therefore, it is not necessary to use a prioritized queue under a global dispatching strategy if the DAG task is trivially schedulable.

Additionally, Lemma~\ref{lem:start_time} allows us to derive the schedulability of a trivially schedulable DAG task under a global \revise{work-conserving} dispatching strategy:

\begin{lemma}\label{lem:global_sufficient_condition}
    If a DAG task $(\mathcal{G},D)$ is \textit{trivially schedulable} on $M$ processors, then it is \textit{schedulable} on $M$ processors using any global \revise{work-conserving} dispatching strategies.
\end{lemma}
\begin{proof}
    By Lemma~\ref{lem:start_time}, we know each sub-job can find at least an idle processor to start its execution at its ready time. Thus, the response time of the sink node is given by the length of the critical path in $\mathcal{G}$. By Definition~\ref{def:trivial_sched}, we know $\mathrm{L}(\mathcal{G}) \leq D$. Therefore, the sink node finishes its execution no later than the task deadline. Since the sink node is a descendant of all other nodes, all the nodes finish their execution by the task deadline. 
\end{proof}

Second, we show that a trivially \revise{schedulable} DAG task is also schedulable using the partitioned dispatching strategy in \revise{Algorithm~\ref{alg:partition}}.

{

\begin{algorithm} [htb]
    \caption{\revise{Partitioned dispatching strategy}}
    \label{alg:partition}

    \KwIn{($\mathcal{G},D)$: a trivially schedulable DAG task on $M$ processors\;}
    \KwOut{a partitioned schedule\;}
    Split $\mathcal{G}$ into $M$ paths $\{p_1,...,p_M\}$ using a minimum path cover algorithm (\eg, \cite{hopcroft1973n})\;
    \For{$k \gets 1$ to $M$}
    {
        Assign the nodes in $p_k$ to processor $k$ and specify their execution order according to the precedence constraints in $\mathcal{G}$\;
    }
\end{algorithm}
}

\begin{lemma}\label{lem:partitioned_sufficient_condition}
    If a DAG task $(\mathcal{G},D)$ is \textit{trivially schedulable} on $M$ processors, then it is \textit{schedulable} on $M$ processors using the partitioned dispatching strategy in \revise{Algorithm~\ref{alg:partition}}.
\end{lemma}
\begin{proof}
    By Definition~\ref{def:trivial_sched}, since $(\mathcal{G},D)$ is trivially schedulable on $M$ processors, we know $\mathrm{L}(\mathcal{G}) \leq D$ and $\mathrm{W}(\mathcal{G}) \leq M$. Since $\mathrm{W}(\mathcal{G}) \leq M$, there exist $M$ paths $\{p_1,...,p_{M}\}$ that can cover all the nodes in $\mathcal{G}$ \revise{(line 1, Algorithm~\ref{alg:partition})}. Since $\mathrm{L}(\mathcal{G}) \leq D$, the length of each path in $\{p_1,...,p_{M}\}$ is smaller than or equal to $\mathrm{L}(\mathcal{G}) \leq D$. Therefore, we can construct a feasible schedule of task $(\mathcal{G},D)$ by assigning each path in $\{p_1,...,p_{M}\}$ to a unique processor, where the execution order is determined according to the precedence constraints \revise{(lines 2-3, Algorithm~\ref{alg:partition})}. 
\end{proof}

Lemma~\ref{lem:partitioned_sufficient_condition} not only proves the schedulability of a trivially schedulable DAG task, but also provides a way to generate a feasible assignment of nodes to processors that can be used by the partitioned dispatching strategy. We use the following example to show the generation process.

\begin{example}\label{eg:trivial_sched}
The trivial schedule of the DAG task in Fig.~\ref{fig:schedule_1} (left) on $M=3$ processors is illustrated in Fig.~\ref{fig:schedule_1} (right). The DAG is split into $3$ paths: $p_1=\{v_1, v_2, v_5, v_7\}$, $p_2=\{v_3,v_6\}$, and $p_3=\{v_4\}$, which are mapped to $P_1, P_2$, and $P_3$, respectively.
\end{example}

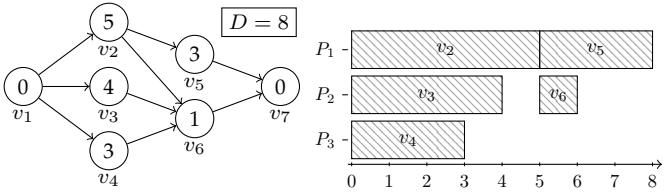
\begin{figure}[t]

\begin{subfigure}{.45\linewidth}
    \centering
    \begin{tikzpicture}[scale=0.57]
        [
            ->,>=stealth',shorten >=1pt,auto,semithick,node font=\footnotesize
        ]
        \node (1) [circle, draw, label={[yshift=-0.83cm]$v_1$}]  at (0.0,1.5) {0};
        \node (2) [circle, draw, label={[yshift=-0.83cm]$v_2$}]  at (2.0,3.0) {5};
        \node (3) [circle, draw, label={[yshift=-0.83cm]$v_3$}]  at (2.0,1.5) {4};
        \node (4) [circle, draw, label={[yshift=-0.83cm]$v_4$}]  at (2.0,0.0) {3};
        \node (5) [circle, draw, label={[yshift=-0.83cm]$v_5$}]  at (4.0,2.25) {3};
        \node (6) [circle, draw, label={[yshift=-0.83cm]$v_6$}]  at (4.0,0.75) {1};
        \node (7) [circle, draw, label={[yshift=-0.83cm]$v_7$}]  at (6.0,1.5) {0};
        \path [->] (1) edge (2);
        \path [->] (1) edge (3);
        \path [->] (1) edge (4);
        \path [->] (2) edge (5);
        \path [->] (2) edge (6);
        \path [->] (3) edge (6);
        \path [->] (4) edge (6);
        \path [->] (5) edge (7);
        \path [->] (6) edge (7);
        \node[draw] at (5.5, 3.0) {$D=8$};
    \end{tikzpicture}
\end{subfigure}
\begin{subfigure}{.45\linewidth}
    \centering
    \begin{tikzpicture}[scale=0.5]

    \draw (-0.25,0.5) -- (-0.1,0.5) node[xshift=-0.30cm,scale=0.7] {$P_3$};
    \draw (-0.25,1.7) -- (-0.1,1.7) node[xshift=-0.30cm,scale=0.7] {$P_2$};
    \draw (-0.25,2.9) -- (-0.1,2.9) node[xshift=-0.30cm,scale=0.7] {$P_1$};

    \draw[->] (-0.1,-0.15) -- (8.25,-0.15);
    \foreach \x in {0,...,8}
        \draw (\x,-0.20) -- (\x,-0.1) node[yshift=-0.25cm,scale=0.7] {$\x$};

    \draw[pattern color=gray!70,pattern=north west lines]   (0.0,2.4) rectangle (5.0,3.4) node[pos=.5,scale=0.75] {$v_2$};
    \draw[pattern color=gray!70,pattern=north west lines]   (5.0,2.4) rectangle (8.0,3.4) node[pos=.5,scale=0.75] {$v_5$};
    \draw[pattern color=gray!70,pattern=north west lines]   (0.0,1.2) rectangle (4.0,2.2) node[pos=.5,scale=0.75] {$v_3$};
    \draw[pattern color=gray!70,pattern=north west lines]   (5.0,1.2) rectangle (6.0,2.2) node[pos=.5,scale=0.75] {$v_6$};
    \draw[pattern color=gray!70,pattern=north west lines]   (0.0,0.0) rectangle (3.0,1.0) node[pos=.5,scale=0.75] {$v_4$};

    
    \node at (8,0) {};
    
    \end{tikzpicture}
\end{subfigure}

\caption{Example of the trivial schedule with $M=3$.}
\label{fig:schedule_1}
\end{figure}

Now, we use Lemma~\ref{lem:global_sufficient_condition} and \ref{lem:partitioned_sufficient_condition} to derive the following exact schedulability test that can be used in conjunction with global and partitioned dispatching strategies.

\begin{theorem}[Schedulability test]\label{theorem:sched_test}
    A DAG task $(\mathcal{G},D)$ is schedulable on $M$ processors \remove{under global and partitioned dispatching strategies }if and only if there exists a trivially schedulable DAG task $(\mathcal{G}',D)$,
    where $\mathcal{G}' \supseteq_E \mathcal{G}$ (\textit{i.e.}, $\mathset{V}'=\mathset{V}$ and $\mathset{E}' \supseteq \mathset{E}$ in the transitive closures of $\mathcal{G}$ and $\mathcal{G'}$).
\end{theorem}
\revise{
\begin{proof}
    \textbf{Sufficiency}.
    Suppose we have a graph $\mathcal{G}' \supseteq_E \mathcal{G}$, and task $(\mathcal{G}',D)$ is trivially schedulable on $M$ processors. By Lemma~\ref{lem:global_sufficient_condition} (\textit{resp.}, Lemma~\ref{lem:partitioned_sufficient_condition}), we know that task $(\mathcal{G}',D)$ is schedulable using a global (\textit{resp.}, partitioned) dispatching strategy. Since task $(\mathcal{G}',D)$ has the same nodes and deadline with task $(\mathcal{G},D)$, and all the precedence constraints in $\mathcal{G}$ are included in $\mathcal{G}'$ (\emph{i.e.}, $\mathset{E}' \supseteq \mathset{E}$), a feasible schedule of task $(\mathcal{G}',D)$ is also a feasible schedule of task $(\mathcal{G},D)$. Thus, task $(\mathcal{G},D)$ is schedulable on $M$ processors.
    
    \textbf{Necessity}.
    We prove the necessity by showing that a trivially schedulable DAG task $(\mathcal{G}' \supseteq_E \mathcal{G},D)$ exists if task $(\mathcal{G},D)$ is schedulable on $M$ processors. 
    Suppose $\mathcal{S}$ is a static schedule of $(\mathcal{G},D)$ on $M$ processors. The static schedule $\mathcal{S}$ specifies the processor allocation $p(v_i) \in \{1, \dots ,M\}$ and execution starting time $s(v_i), \forall v_i\in\mathcal{V}$. 
    Using the set of edges describing the precedence per processor $k$ as 
    $$\hat{\mathcal{E}}_k = \{e_{ij},~\forall i,j  ~|~p(v_i) = p(v_j) = k \land s(v_i) \leq s(v_j) \}$$
    and 
    $\hat{\mathcal{E}} = \hat{\mathcal{E}}_1 \cup ...\cup \hat{\mathcal{E}}_M$, 
    we can construct the graph $\mathcal{G}' = (\mathcal{V}, \mathcal{E} \cup \hat{\mathcal{E}})$. 
    By construction, the nodes in $\mathcal{G}'$ can be covered by $M$ paths, each of which contains the nodes assigned to a processor $k$, whose precedence is described by $\hat{\mathcal{E}}_k$, yielding $\mathrm{W}(\mathcal{G}') \leq M$. In addition, since $\mathcal{S}$ respects all precedence constraints in $\mathcal{E}$ and $\hat{\mathcal{E}}$, it is a feasible schedule of task $(\mathcal{G}',D)$, thus $\mathrm{L}(\mathcal{G}') \leq D$. Since $\mathrm{W}(\mathcal{G}')\leq M$ and $\mathrm{L}(\mathcal{G}')\leq D$, $\mathcal{G}' \supseteq_E \mathcal{G}$ is trivially schedulable on $M$ processors.
\end{proof}
}

From Theorem \ref{theorem:sched_test}, we know that a feasible schedule of task $(\mathcal{G},D)$ on fewer processors than the current DAG width can be obtained by generating a graph $\mathcal{G}'$ that has additional edges to $\mathcal{G}$ and satisfies both the length and width constraints. 
Example~\ref{eg:trivial_sched_2} shows a DAG generated by adding edge $e_{3,4}$ to the DAG task in Example~\ref{eg:trivial_sched}. As illustrated in \revise{Fig.}~\ref{fig:schedule_2}, the resulting DAG is trivially schedulable on $M=2$ processors (compared to $M=3$ in Example~\ref{eg:trivial_sched}). 
Based on this observation, we propose a new DAG scheduling algorithm called Edge Generation Scheduling (EGS) to minimize processor usage in Section~\ref{sec:egs}.

\begin{example}\label{eg:trivial_sched_2}
The trivial schedule of the DAG task in Fig.~\ref{fig:schedule_2} (left) on $M=2$ processors is illustrated in Fig.~\ref{fig:schedule_2} (right). The DAG is split into $2$ paths: $p_1=\{v_1, v_2, v_5, v_7\}$ and $p_2=\{v_3,v_4,v_6\}$, which are mapped to $P_1$ and $P_2$, respectively.
\end{example}
\begin{figure}[t]
\begin{subfigure}{.45\linewidth}
    \centering
    \begin{tikzpicture}[scale=0.57]
        [
            ->,>=stealth',shorten >=1pt,auto,semithick,node font=\footnotesize
        ]
        \node (1) [circle, draw, label={[yshift=-0.83cm]$v_1$}]  at (0.0,1.5) {0};
        \node (2) [circle, draw, label={[yshift=-0.83cm]$v_2$}]  at (2.0,3.0) {5};
        \node (3) [circle, draw, label={[yshift=-0.83cm]$v_3$}]  at (2.0,1.5) {4};
        \node (4) [circle, draw, label={[yshift=-0.83cm]$v_4$}]  at (2.0,0.0) {3};
        \node (5) [circle, draw, label={[yshift=-0.83cm]$v_5$}]  at (4.0,2.25) {3};
        \node (6) [circle, draw, label={[yshift=-0.83cm]$v_6$}]  at (4.0,0.75) {1};
        \node (7) [circle, draw, label={[yshift=-0.83cm]$v_7$}]  at (6.0,1.5) {0};
        \path [->] (1) edge (2);
        \path [->] (1) edge (3);
        \path [->] (2) edge (5);
        \path [->] (2) edge (6);
        \path [->] (4) edge (6);
        \path [->] (5) edge (7);
        \path [->] (6) edge (7);
        \path [->, Green] (3) edge[bend right=50] (4);
        \node[draw] at (5.5, 3.0) {$D=8$};
    \end{tikzpicture}
\end{subfigure}
\begin{subfigure}{.45\linewidth}
    \centering
    \begin{tikzpicture}[scale=0.5]

    \draw (-0.25,0.5) -- (-0.1,0.5) node[xshift=-0.30cm,scale=0.7] {$P_2$};
    \draw (-0.25,1.7) -- (-0.1,1.7) node[xshift=-0.30cm,scale=0.7] {$P_1$};

    \draw[->] (-0.1,-0.15) -- (8.25,-0.15);
    \foreach \x in {0,...,8}
        \draw (\x,-0.20) -- (\x,-0.1) node[yshift=-0.25cm,scale=0.7] {$\x$};

    \draw[pattern color=gray!70,pattern=north west lines]   (0.0,1.2) rectangle (5.0,2.2) node[pos=.5,scale=0.75] {$v_2$};
    \draw[pattern color=gray!70,pattern=north west lines]   (5.0,1.2) rectangle (8.0,2.2) node[pos=.5,scale=0.75] {$v_5$};
    \draw[pattern color=gray!70,pattern=north west lines]   (0.0,0.0) rectangle (4.0,1.0) node[pos=.5,scale=0.75] {$v_3$};
    \draw[pattern color=gray!70,pattern=north west lines]   (4.0,0.0) rectangle (7.0,1.0) node[pos=.5,scale=0.75] {$v_4$};
    \draw[pattern color=gray!70,pattern=north west lines]   (7.0,0.0) rectangle (8.0,1.0) node[pos=.5,scale=0.75] {$v_6$};

    \node at (8,0) {};
    \end{tikzpicture}
\end{subfigure}
\caption{Example of the trivial schedule with $M=2$.}
\label{fig:schedule_2}
\end{figure}
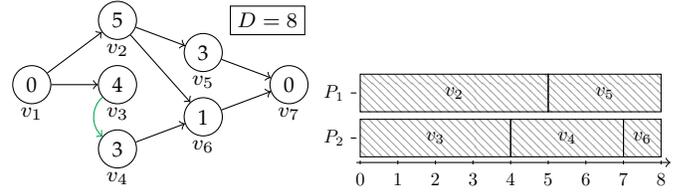

\section{Edge Generation Scheduling}
\label{sec:egs}
We propose the edge generation scheduling (EGS) framework based on trivial schedulability. Consider a common DAG scheduling problem: is a given DAG task $(\mathcal{G},D)$ schedulable on $M$ processors? For EGS, we reformulate the question: what is the minimum number of processors needed to schedule the DAG task $(\mathcal{G},D)$? Using trivial schedulability, the question forms the optimization problem:
\begin{alignat}{2}
    & \underset{\mathcal{G}^\prime \supseteq_E \mathcal{G}}{\text{minimize}} ~~~~~~~&&{\mathrm{W}(\mathcal{G}^\prime)} \\
    &\text{subject to} &&\mathrm{L}(\mathcal{G}^\prime) \leq D \nonumber
\end{alignat}
\textit{i.e.}, finding a graph $\mathcal{G}^\prime~ \supseteq_E ~\mathcal{G}$ that has minimal width while maintaining the length constraint. We aim to find this $\mathcal{G}^\prime$ by iteratively adding edges until no edges can be added without violating the length constraint or until a lower bound on the width is reached.

The rest of the section describes which edges can be added in Section \ref{sec:action_mask} and how to compute the lower bound in Section \ref{sec:lb}. Section \ref{sec:algorithm} shows the algorithm and complexity of EGS, and Section \ref{sec:example} shows an example to highlight the challenge in the edge selection choice.

\subsection{Eligible edges}\label{sec:action_mask}
We define edge masks as Boolean matrices, where each entry is active if the corresponding edge can be added to the current DAG and inactive otherwise. Four different edge masks are developed for different purposes: \textit{redundancy mask}, \textit{cycle mask}, \textit{length mask}, and \textit{width mask}. 

\begin{itemize}
    \item \textit{Redundancy mask} $\matr{M_r}$. The redundancy mask avoids adding redundant edges already existing in the current graph. The redundancy mask can be calculated as the reverse of the graph's transitive closure:
    \begin{equation} \label{eq:r_mask}
        \matr{M_r} = \lnot \matr{T_c}
    \end{equation}
    
    \item \textit{Cycle mask} $\matr{M_c}$. The cycle mask avoids introducing cycles when adding edges to the graph. It can be computed as
    \begin{equation}\label{eq:c_mask}
        \matr{M_c} = \lnot (({\matr{T_c}})^\intercal \lor \mathbf{I}),
    \end{equation}
    where $({\matr{T_c}})^\intercal$ is the transpose of the transitive closure and $\mathbf{I}$ denotes the Boolean identity matrix.
    
    \item \textit{Length mask} $\matr{M_l}$. The length mask avoids adding edges that would violate the length constraint. The length mask can be derived by comparing the EFT of the predecessor node $v_i$ and the LST of the successor node $v_j$ of the candidate edge $e_{ij}$:
    \begin{equation}\label{eq:l_mask}
        [\matr{M_l}]_{ij} = \mathbbm{1}_{\revise{t}^\text{EFT}_i \leq \revise{t}^\text{LST}_j} ,\quad \forall i,j=1,...,n 
    \end{equation}
    where $\mathbbm{1}_{\revise{t}^\text{EFT}_i \leq \revise{t}^\text{LST}_j}$ is an indicator function, returning $1$ if $\revise{t}^\text{EFT}_i \leq \revise{t}^\text{LST}_j$ and $0$ otherwise. 
    \revise{The correctness of the length mask is proved by Theorem~\ref{theorem:length_const}.
    }

    \item \textit{Width mask} $\matr{M_w}$. The width mask restricts edges to be generated between the nodes with the largest lateral width, which is a necessary condition to reduce DAG width as shown in Theorem~\ref{the:necessary_condition}. 
    \revise{
    \begin{equation}\label{eq:w_mask}
        [\matr{M_w}]_{ij} = \mathbbm{1}_{\mathbf{m}^\text{LW}_i = \mathbf{m}^\text{LW}_j = \mathrm{W}(\mathcal{G}) - 1} ,\quad \forall i,j=1,...,n,
    \end{equation}
    where $\mathbf{m}^\text{LW}_i$ denotes the lateral width of node $v_i$. 
    }
\end{itemize}

\revise{
\begin{theorem}[Length constraint]
\label{theorem:length_const}
    Given a DAG task $(\mathcal{G},D)$ with $\mathrm{L}(\mathcal{G}) \leq D$ and DAG $\mathcal{G}'=(\mathcal{V}, \mathcal{E} \cup \{e_{ij}\})$, with $v_j \notin des(v_i)$ (redundant) and $v_j \notin anc(v_i)\cup\{v_i\}$ (cycle), $\mathrm{L}(\mathcal{G}') \leq D$ if and only if $t^\text{EFT}_i \leq t^\text{LST}_j$ in task $(\mathcal{G},D)$. 
\end{theorem}
\begin{proof}   
    Adding an edge $e_{ij}$ to $\mathcal{G}$ to create $\mathcal{G}'$ does not change $t_i^\text{EFT}$ and $t_j^\text{LST}$ as it does not alter the ancestors and descendants of $v_i$ and $v_j$, respectively.
    By definition, $t^\text{EFT}_i$ equals the length of the longest path between the source node and $v_i$ (denoted by $p^{l}_i, \mathrm{L}(p^{l}_i)=t^\text{EFT}_i$). 
    Conversely, $t^\text{LST}_j$ equals the difference of $D$ minus the length of the longest path between $v_j$ and the sink node (denoted by $p_j^{r}, \mathrm{L}(p^{r}_j)=D-t^\text{LST}_j$).
    Since $e_{ij}$ connects $v_i$ and $v_j$ in $\mathcal{G}'$, there exists a path $\hat{p}_{ij} = p_i^l \cup p_j^r$ with $\mathrm{L}(\hat{p}_{ij}) = t^\text{EFT}_i + D - t^\text{LST}_j$, which is the longest path among all the paths that go through $v_i$ and $v_j$. Therefore,
    $$\mathrm{L}(\mathcal{G}') \geq \mathrm{L}(\hat{p}_{ij}) = t^\text{EFT}_i + D - t^\text{LST}_j.$$
    The inequality is tight when $\hat{p}_{ij}$ is a critical path of $\mathcal{G}'$.
    
    \textbf{Sufficiency.}
    Denote $\mathcal{P}$ as set of paths of $\mathcal{G}$ and $\mathcal{P}^{*\prime}$ as the set of critical paths of $\mathcal{G}'$.
    There are two cases to consider:
    \begin{itemize}        
        \item $\hat{p}_{ij} \in \mathcal{P}^{*\prime}$. The above inequality is tight, \ie, $\mathrm{L}(\mathcal{G}') = t^\text{EFT}_i + D - t^\text{LST}_j$. It follows $t^\text{EFT}_i \leq t^\text{LST}_j \implies \mathrm{L}(\mathcal{G}') \leq D$.
        \item $\hat{p}_{ij} \notin \mathcal{P}^{*\prime}$. Since $\hat{p}_{ij}$ is the longest path among all the paths that go through $v_i$ and $v_j$, it follows that either $v_i \notin p^{*\prime}$ or $v_j \notin p^{*\prime}, \forall p^{*\prime} \in \mathcal{P}^{*\prime}$. Therefore, $\mathcal{P}^{*\prime} \subseteq \mathcal{P}$,
        $\mathrm{L}(\mathcal{G}') = \mathrm{L}(\mathcal{G}) \leq D$.
    \end{itemize}
    
    \textbf{Necessity.}
    $\mathrm{L}(\mathcal{G}') \leq D \land \mathrm{L}(\mathcal{G}') \geq t^\text{EFT}_i + D - t^\text{LST}_j \implies t^\text{EFT}_i + D - t^\text{LST}_j \leq D \implies t^\text{EFT}_i \leq t^\text{LST}_j$.
\end{proof}
}

\begin{theorem}[Width reduction]
\label{the:necessary_condition}
For any DAG $\mathcal{G}' \supseteq_E \mathcal{G}$, $\mathrm{W}(\mathcal{G}') < \mathrm{W}(\mathcal{G})$ only if the transitive closure of $\mathcal{G}'$ has at least one edge between the nodes with the largest lateral width in $\mathcal{G}$.
\end{theorem}
\begin{proof}
We prove the theorem by contradiction. Suppose there exists a DAG $\mathcal{G}' \supseteq_E \mathcal{G}$ whose width is lower than $\mathcal{G}$ and has no edge between the nodes with the largest LW in $\mathcal{G}$. We denote the critical antichain of DAG $\mathcal{G}$ as $q$ (\textit{i.e.}, $|q|=\mathrm{W}(\mathcal{G})$). Since $\mathcal{G}'$ has no edge between the nodes with the largest LW, there is no edge between the nodes belonging to the critical antichain $q$. Thus, $q$ is an antichain of $\mathcal{G}'$. By the definition of the DAG width, we know that the width of $\mathcal{G}'$ is equal to or larger than $|q|=\mathrm{W}(\mathcal{G})$, which contradicts our assumption.
\end{proof}

Finally, the complete edge mask is obtained by combining all the above edge masks:
\begin{equation}\label{eq:action_mask_full}
    \matr{M} = \matr{M_r} \land \matr{M_c} \land \matr{M_l} \land \matr{M_w}
\end{equation}
\remove{
Furthermore, we observe that edges that are redundant are edges to descendants; edges that add cycles are edges to ancestors. Since the width mask only allows edges within the antichain, it excludes dependents and ancestors and thus can replace these masks (\textit{i.e.}, $\matr{M_w} \land \matr{M_r} \land \matr{M_c} = \matr{M_w}$). Therefore the complete mask can be reduced to }
\removeeq{11}{\matr{M} = \matr{M_l} \land \matr{M_w}}
\remove{
For ease of understanding, Fig. \ref{fig:eg_act_mask} shows one example for each type of edge masks in Equation~\eqref{eq:action_mask_full}.
}
\begin{example}
The valid and invalid edges of the DAG task in Example \ref{eg:dag} are illustrated in Fig.~\ref{fig:eg_act_mask}.
\end{example}

\usetikzlibrary{automata, chains, quotes}

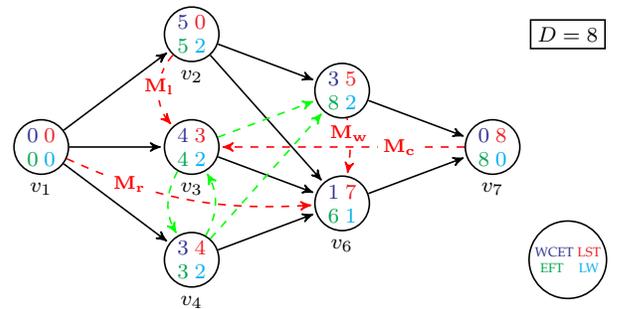
\begin{figure}[ht]
\centering
\begin{tikzpicture}
    [
        ->,>=stealth',shorten >=1pt,auto,node distance=2.8cm,semithick,node font=\footnotesize
    ]
    \tikzset
    {
        in place/.style=
        {
          auto=false,
          fill=white,
          inner sep=2pt,
        },
    }
    \node (1) [circle, draw, align=center, label=below:$v_1$, inner sep=1pt] at (0.0,1.5) 
    {
        $\textcolor{\cwcet}{0}$ $\textcolor{\clst}{0}$ \\ $\textcolor{\ceft}{0}$ $\textcolor{\clw}{0}$
    };
    \node (2) [circle, draw, align=center, label=below:$v_2$, inner sep=1pt] at (2.0,3.0) 
    {
        $\textcolor{\cwcet}{5}$ $\textcolor{\clst}{0}$ \\ $\textcolor{\ceft}{5}$ $\textcolor{\clw}{2}$
    };
    \node (3) [circle, draw, align=center, label=below:$v_3$, inner sep=1pt]  at (2.0,1.5)
    {
        $\textcolor{\cwcet}{4}$ $\textcolor{\clst}{3}$ \\ $\textcolor{\ceft}{4}$ $\textcolor{\clw}{2}$
    };
    \node (4) [circle, draw, align=center, label=below:$v_4$, inner sep=1pt]  at (2.0,0.0)
    {
        $\textcolor{\cwcet}{3}$ $\textcolor{\clst}{4}$ \\ $\textcolor{\ceft}{3}$ $\textcolor{\clw}{2}$
    };
    \node (5) [circle, draw, align=center, label=below:$v_5$, inner sep=1pt]  at (4.0,2.25)
    {
        $\textcolor{\cwcet}{3}$ $\textcolor{\clst}{5}$ \\ $\textcolor{\ceft}{8}$ $\textcolor{\clw}{2}$
    };
    \node (6) [circle, draw, align=center, label=below:$v_6$, inner sep=1pt]  at (4.0,0.75)
    {
        $\textcolor{\cwcet}{1}$ $\textcolor{\clst}{7}$ \\ $\textcolor{\ceft}{6}$ $\textcolor{\clw}{1}$
    };
    \node (7) [circle, draw, align=center, label=below:$v_7$, inner sep=1pt]  at (6.0,1.5)
    {
        $\textcolor{\cwcet}{0}$ $\textcolor{\clst}{8}$ \\ $\textcolor{\ceft}{8}$ $\textcolor{\clw}{0}$
    };

    \node (legend) [circle, draw, align=center, label=below:, inner sep=1pt, font=\tiny]  at (7.0,0.0)
    {
        \textcolor{\cwcet}{WCET} \textcolor{\clst}{LST} \\ \textcolor{\ceft}{~EFT} \textcolor{\clw}{~~~LW}
    };


    \path [->, red, dashed] (1) edge [bend right=14] node[pos=0.25, in place, node font=\scriptsize] {$\matr{M_r}$} (6);
    \path [->, red, dashed] (5) edge [bend left=10] node[pos=0.25, in place, node font=\scriptsize] {$\matr{M_w}$} (6);
    \path [->, red, dashed] (2) edge [bend right=42] node[pos=0.45, in place, node font=\scriptsize] {$\matr{M_l}$} (3);
    \path [->, red, dashed] (7) edge [bend right=0] node[pos=0.25, in place, node font=\scriptsize] {$\matr{M_c}$} (3);
    
    \path [->] (1) edge (2);
    \path [->] (1) edge (3);
    \path [->] (1) edge (4);
    \path [->] (2) edge (5);
    \path [->] (2) edge (6);
    \path [->] (3) edge (6);
    \path [->] (4) edge (6);
    \path [->] (5) edge (7);
    \path [->] (6) edge (7);

    \path [->, green, dashed] (4) edge [bend right] (3);
    \path [->, green, dashed] (3) edge [bend right] (4);
    \path [->, green, dashed] (3) edge (5);
    \path [->, green, dashed] (4) edge (5);

    \node[draw] at (7.0, 3.0) {$D=8$};
    
\end{tikzpicture}
\caption{Example of the valid and invalid edges, shown with the green and red dashed edges, respectively. For ease of presentation, we only show one invalid action masked out by each type of action mask
($\matr{M_r}$, $\matr{M_c}$, $\matr{M_l}$ or $\matr{M_w}$). 
The numbers inside each node represent the node's WCET, LST, EFT, and LW with corresponding colors.}
\label{fig:eg_act_mask}
\end{figure}%

\subsection{Lower bound of task parallelism} \label{sec:lb}
\begin{theorem}[Lower bound] \label{theorem:lb}
\revise{Given a DAG task $(\mathcal{G},D)$, a set of sub-tasks $\widehat{\mathset{V}} \subseteq \mathset{V}$ are not schedulable on fewer than $\mathrm{LB}(\widehat{\mathset{V}})$ processors:}
\begin{equation} \label{eq:sub_lb}
    \mathrm{LB}(\widehat{\mathset{V}}) = \left \lceil \frac{\sum_{v_i \in \widehat{\mathset{V}}}{C_i}}{\max\limits_{v_i \in \widehat{\mathset{V}}}{\revise{t}^\text{LFT}_i} - \min\limits_{v_j \in \widehat{\mathset{V}}}{\revise{t}^\text{EST}_j}} \right \rceil
\end{equation}
\end{theorem}
\begin{proof}
We prove the theorem by contradiction. Assume all the sub-tasks in $\widehat{\mathset{V}}$ are schedulable on $M' \leq \mathrm{LB}(\widehat{\mathset{V}}) - 1$ processors. Then, the maximum response time among all the sub-tasks in $\widehat{\mathset{V}}$ is at least $\sum_{v_i \in \widehat{\mathset{V}}}{C_i} / M' \geq \sum_{v_i \in \widehat{\mathset{V}}}{C_i} / 
(\mathrm{LB}(\widehat{\mathset{V}}) - 1)
> \max\limits_{v_i \in \widehat{\mathset{V}}}{\revise{t}^\text{LFT}_i} - \min\limits_{v_j \in \widehat{\mathset{V}}}{\revise{t}^\text{EST}_j}$. This implies that some sub-task cannot finish its execution within the required finishing time, which contradicts the assumption that all the sub-tasks in $\widehat{\mathset{V}}$ are schedulable.
\end{proof}

As $\mathrm{LB}(\widehat{\mathset{V}})$ gives the lower bound of the number of processors required by the subset of task nodes $\widehat{\mathset{V}}$, it is also a valid lower bound of the DAG task's parallelism $\underline{M}$.
By computing the $\mathrm{LB}(\widehat{\mathset{V}})$ of each node subset $\widehat{\mathset{V}}$, we could derive a tighter $\underline{M}$ by taking the maximum $\mathrm{LB}(\widehat{\mathset{\mathset{V}}})$ among all possible node subsets (\textit{i.e.}, $\underline{M} = \max_{\widehat{V}}{\mathrm{LB}(\widehat{\mathset{V}})}$). 
However, since the total number of node subsets is exponential with respect to the number of nodes in the DAG, deriving the lower bound of every possible node subset is not computationally tractable.
In this paper, we consider two special cases of $\widehat{\mathset{V}}$: (i) all the nodes in the DAG (\textit{i.e.}, $\widehat{\mathset{V}} = \mathset{V}$), (ii) the nodes with the largest lateral width (\textit{i.e.}, $\widehat{\mathset{V}} = \{v_i \in \mathset{V} | m^\text{LW}_i = \mathrm{W}(\mathcal{G})-1\}$), and take the larger $\mathrm{LB}(\widehat{\mathset{V}})$ as the lower bound of the task parallelism $\underline{M}$, as shown in Corollary~\ref{cor:lb}.

\begin{corollary} \label{cor:lb}
\revise{
Given a DAG task $(\mathcal{G},D)$, the task is not schedulable on fewer than $\underline{M}$ processors:}
\begin{equation} \label{eq:lb}
    \underline{M} = \max\{\mathrm{LB}(\mathset{V}), \mathrm{LB}(\{v_i \in \mathset{V} | m^\text{LW}_i = \revise{\mathrm{W}(\mathcal{G})}-1\})\}
\end{equation}
\end{corollary}

\subsection{EGS framework}\label{sec:algorithm}

Algorithm \ref{alg:egs} shows specific procedures of the EGS framework. In line 1, a DAG $\mathcal{G}'$ is initialized as the input graph $\mathcal{G}$. The width of the DAG and its lower bound are computed accordingly. In line 2, the edge mask of DAG $\mathcal{G}'$ is initialized using Equations \eqref{eq:r_mask}-\eqref{eq:action_mask_full}. In lines 3-6, edges are added iteratively to DAG $\mathcal{G}'$ until (i) no edges can be added according to the edge mask $\matr{M}$ or (ii) the current DAG width reaches its lower bound. In each iteration, an edge is selected according to an edge generation policy $\pi$ (line 4). The policy takes the current DAG as input and outputs an edge $e_{ij}$ to be added in line 5. The generated edge must comply with the edge mask such that $[\matr{M}]_{ij} = 1$. 
At the end of each iteration,
The DAG width and its lower bound are updated in line 6, and 
the edge mask is updated in line 7. 

{
\begin{algorithm} [htb]
    \caption{The EGS framework}
    \label{alg:egs}
    
    \KwIn{($\mathcal{G},D)$: the DAG task to be scheduled\;}
    \KwOut{$\mathcal{G}'$: the DAG with minimized width\;}
    $\mathcal{G}' \leftarrow \mathcal{G}$, $\underline{M} \leftarrow \mathrm{LB}(\mathcal{G}')$\;
    Compute edge mask $\matr{M}$ according to Equation (\ref{eq:action_mask_full})\;
    \While{$\lor_{i=1}^n \lor_{j=1}^n [\matr{M}]_{ij}$ \And $\underline{M} < \mathrm{W}(\mathcal{G})$}
    {
        $e_{ij} \gets $ Select with policy $\pi$ and edge mask $\matr{M}$\;
        $\mathcal{G}' \leftarrow \mathcal{G}' \cup \{e_{ij}\}$\;
        $\underline{M} \leftarrow \mathrm{LB}(\mathcal{G}')$\;
        Update edge mask $\matr{M}$\;
    }
    \Return $\mathcal{G}'$\;
\end{algorithm}
}

    

The time complexity of the proposed EGS framework is analyzed as follows. In each iteration of EGS, the time complexity depends on the complexity of updating the edge masks and selecting the edge to be added to the graph. Recall that we consider four edge masks. Updating each of them requires the update of the transitive closure, which takes a complexity of $
\mathcal{O}(n^2)$. Once we get the updated transitive closure, the redundancy mask and the cycle mask can be updated within $
\mathcal{O}(n^2)$. The update of the length and width mask requires recomputing each node's timing and parallelism attributes from scratch, which takes a time complexity of $\mathcal{O}(n^3)$ and $\mathcal{O}(n \sqrt{n} m^*)$, respectively (details see Section \ref{sec:algebra}). Thus, the overall complexity of updating the edge masks is bounded by $\mathcal{O}(n^3 + n\sqrt{n} \cdot m^*)$, where $m^*$ denotes the number of edges in the transitive closure. Since in the worst case at most 
$n^2$
edges can be added to the original graph, the time complexity of the EGS algorithm is bounded by $\mathcal{O}(n^2 \cdot (n^3 + n\sqrt{n} \cdot m^* + \Omega_\pi))$, where $\Omega_\pi$ is the time complexity of the edge generation policy $\pi$.

\subsection{Example of edge generation policy}\label{sec:example}

So far, we have not discussed specific edge generation policies that can be used within the proposed EGS framework.
Here, we use an example to illustrate two different edge generation policies, which result in two different schedules with different processor usage. This example indicates that the edge generation policy is critical to the scheduling performance of EGS and motivates us to develop a deep reinforcement learning algorithm to learn an efficient edge generation policy in Section~\ref{sec:drl}.

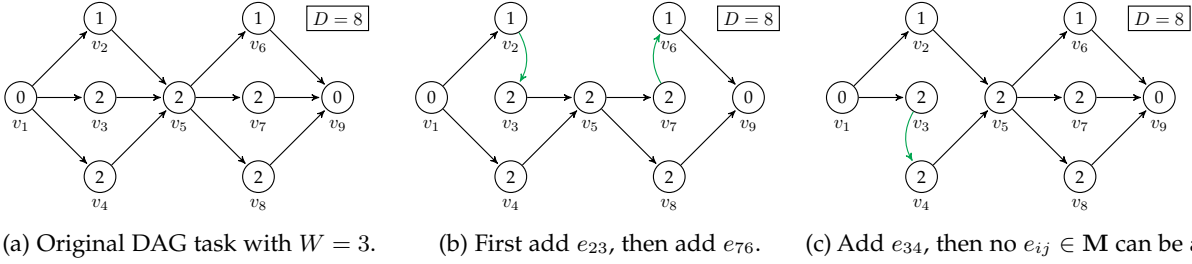
\begin{figure*}[t]
\centering
\begin{subfigure}{.3\textwidth}
\centering
\resizebox{0.9\textwidth}{!}
{%
\begin{tikzpicture}
    [
        ->,>=stealth',shorten >=1pt,auto,node distance=2.8cm,semithick
    ]
    \node (1) [circle, draw, label=below:$v_1$]  at (0.0,1.5) {0};
    \node (2) [circle, draw, label=below:$v_2$]  at (1.5,3.0) {1};
    \node (3) [circle, draw, label=below:$v_3$]  at (1.5,1.5) {2};
    \node (4) [circle, draw, label=below:$v_4$]  at (1.5,0.0) {2};
    \node (5) [circle, draw, label=below:$v_5$]  at (3.0,1.5) {2};
    \node (6) [circle, draw, label=below:$v_6$]  at (4.5,3.0) {1};
    \node (7) [circle, draw, label=below:$v_7$]  at (4.5,1.5) {2};
    \node (8) [circle, draw, label=below:$v_8$]  at (4.5,0.0) {2};
    \node (9) [circle, draw, label=below:$v_9$]  at (6.0,1.5) {0};
    \path [->] (1) edge (2);
    \path [->] (1) edge (3);
    \path [->] (1) edge (4);
    \path [->] (2) edge (5);
    \path [->] (3) edge (5);
    \path [->] (4) edge (5);
    \path [->] (5) edge (6);
    \path [->] (5) edge (7);
    \path [->] (5) edge (8);
    \path [->] (6) edge (9);
    \path [->] (7) edge (9);
    \path [->] (8) edge (9);

    \node[draw] at (6.0, 3.0) {$D=8$};
\end{tikzpicture}
}
\caption{Original DAG task with $W=3$.}
\label{fig:egs_0}
\end{subfigure}%
\begin{subfigure}{.3\textwidth}
\centering
\resizebox{0.9\textwidth}{!}
{%
\begin{tikzpicture}
    [
        ->,>=stealth',shorten >=1pt,auto,node distance=2.8cm,semithick
    ]
    \node (1) [circle, draw, label=below:$v_1$]  at (0.0,1.5) {0};
    \node (2) [circle, draw, label=below:$v_2$]  at (1.5,3.0) {1};
    \node (3) [circle, draw, label=below:$v_3$]  at (1.5,1.5) {2};
    \node (4) [circle, draw, label=below:$v_4$]  at (1.5,0.0) {2};
    \node (5) [circle, draw, label=below:$v_5$]  at (3.0,1.5) {2};
    \node (6) [circle, draw, label=below:$v_6$]  at (4.5,3.0) {1};
    \node (7) [circle, draw, label=below:$v_7$]  at (4.5,1.5) {2};
    \node (8) [circle, draw, label=below:$v_8$]  at (4.5,0.0) {2};
    \node (9) [circle, draw, label=below:$v_9$]  at (6.0,1.5) {0};
    \path [->] (1) edge (2);
    \path [->] (1) edge (4);
    \path [->] (3) edge (5);
    \path [->] (4) edge (5);
    \path [->] (5) edge (7);
    \path [->] (5) edge (8);
    \path [->] (6) edge (9);
    \path [->] (8) edge (9);

    \path [->, Green] (2) edge[bend left] (3);
    \draw [->, Green] (7) edge[bend left] (6);

    \node[draw] at (6.0, 3.0) {$D=8$};
\end{tikzpicture}
}
\caption{First add $e_{23}$, then add $e_{76}$.}
\label{fig:egs_1}
\end{subfigure}%
\begin{subfigure}{.3\textwidth}
\centering
\resizebox{0.9\textwidth}{!}
{%
\begin{tikzpicture}
    [
        ->,>=stealth',shorten >=1pt,auto,node distance=2.8cm,semithick
    ]
    \node (1) [circle, draw, label=below:$v_1$]  at (0.0,1.5) {0};
    \node (2) [circle, draw, label=below:$v_2$]  at (1.5,3.0) {1};
    \node (3) [circle, draw, label=below:$v_3$]  at (1.5,1.5) {2};
    \node (4) [circle, draw, label=below:$v_4$]  at (1.5,0.0) {2};
    \node (5) [circle, draw, label=below:$v_5$]  at (3.0,1.5) {2};
    \node (6) [circle, draw, label=below:$v_6$]  at (4.5,3.0) {1};
    \node (7) [circle, draw, label=below:$v_7$]  at (4.5,1.5) {2};
    \node (8) [circle, draw, label=below:$v_8$]  at (4.5,0.0) {2};
    \node (9) [circle, draw, label=below:$v_9$]  at (6.0,1.5) {0};
    \path [->] (1) edge (2);
    \path [->] (1) edge (3);
    \path [->] (2) edge (5);
    \path [->, Green] (3) edge[bend right] (4);
    \path [->] (4) edge (5);
    \path [->] (5) edge (6);
    \path [->] (5) edge (7);
    \path [->] (5) edge (8);
    \path [->] (6) edge (9);
    \path [->] (7) edge (9);
    \path [->] (8) edge (9);

    \node[draw] at (6.0, 3.0) {$D=8$};
\end{tikzpicture}
}
\caption{\mbox{Add $e_{34}$, then no $e_{ij}\in\matr{M}$ can be added.}}
\label{fig:egs_2}
\end{subfigure}
\caption{Example of different edge generation policies. Given a DAG task $(\mathcal{G},D)$ with $\mathcal{G}$ illustrated in (a) and $D=8$, the edge generation policy used in (b) reduces the DAG width to $2$, while the policy used in (c) cannot reduce the DAG width.}
\label{fig:egs_example}
\end{figure*}

\begin{example}
\label{eg:egs}
We apply the EGS with two different edge generation policies to the DAG task in Fig.~\ref{fig:egs_0} and compare their scheduling results. In the first iteration of the EGS, the two policies select to add edge $e_{43}$ and $e_{35}$, resulting in the DAGs shown in Fig. \ref{fig:egs_1} and Fig. \ref{fig:egs_2}, respectively.
\end{example}

\section{Deep Reinforcement Learning}
\label{sec:drl}

Finding an optimal edge generation policy is challenging. Recall that EGS aims to reduce the DAG width by adding one edge in each step. Since the decision made in one step impacts the following decision process, a greedy policy that always selects the edge with the maximum intermediate width reduction may lead to a local optimum. \secondr{Other simple heuristics also often struggle with the global complexity of NP-hard problems. In this section, we formulate the edge generation as a Markov Decision Process (MDP) that aims to maximize the width reduction. Then, we use the DRL algorithm Proximal Policy Optimization (PPO)~\cite{schulman2017proximal} to find a policy with good global performance for this NP-hard MDP.} 

\subsection{MDP formulation}
An MDP is defined through the tuple $(\mathcal{S}, \mathcal{A}, \mathrm{R}, \mathrm{T}, \gamma)$ with the state space $\mathcal{S}$ and the action space $\mathcal{A}$. A reward function $\mathrm{R}:\mathcal{S}\times\mathcal{A} \mapsto \mathbb{R}$ assigns a scalar reward to state-action pairs. In a deterministic MDP such as used in this work, a transition function $\mathrm{T}:\mathcal{S}\times\mathcal{A}\mapsto\mathcal{S}$ determines the next state according to the current state and current action. A discount factor $\gamma\in [0,1)$ balances the importance of immediate and future rewards. The goal in an MDP is to find a policy $\pi:\mathcal{S}\times\mathcal{A}\mapsto \mathbb{R}^+$ that assigns a probability to each action given a state. The policy aims to maximize the expected cumulative discounted return
\begin{equation}
    \mathcal{R}^\pi(s)=\mathbb{E}_{a\sim \pi(a|s)}\left[\sum_{t=0}^\infty \gamma^t \mathrm{R}(s_{t}, a_{t})~|~s_0 = s\right].
\end{equation}

\revise{In the EGS framework, the state $s_t\in \mathcal{S}$ is a DAG $\mathcal{G}_t = (\mathcal{V}, \mathcal{E}_t)$ with its constant vertices and increasing set of edges. The initial state is the DAG to be scheduled, \ie, $\mathcal{G}_0 = \mathcal{G}$. The action space $\mathcal{A}$ is defined as all eligible edges according to the action mask $\matr{M}_t$, of which the policy can choose an edge $a_t\in \matr{M}_t$ to add to the graph. Thus, the transition function $\mathrm{T}$ results in $\mathcal{E}_{t+1} = \mathcal{E}_t \cup \{a_t\}$, and $\matr{M}_{t+1}$ is updated according to the new edges using \eqref{eq:action_mask_full}. In EGS, the goal is to minimize the number of cores needed to schedule the DAG. Therefore, the reward is defined as
\begin{equation}
    \mathrm{R}(s_t, a_t) = \mathrm{W}(\mathcal{G}_t) - \mathrm{W}(\mathcal{G}_{t+1}),
\end{equation}
giving a reward equal to the width reduction after adding an edge. The MDP terminates if the action mask is empty or the width reaches the lower bound. To aid the agent in solving the MDP we provide it with precomputed node features consisting of (i) the node-level timing attributes, including the WCET, EFT, and LST and (ii) the node-level parallelism attributes, including the LW, IW, and OW.}

\subsection{PPO algorithm}
\label{sec:ppo}

The PPO algorithm \cite{schulman2017proximal} is an on-policy, actor-critic reinforcement learning algorithm that trains a value function $V^\pi_\phi(s)$ (known as critic network) that predicts the cumulative return $\mathcal{R}(s)$ of a state $s$ under the currently active policy $\pi_\theta(s)$ (known as actor-network). Conducting a rollout (\ie, a long sequence of state-action-reward tuples), an advantage of the actions in the rollout is approximated using a generalized advantage estimate (GAE)\cite{schulman2015high}. The advantage indicates whether the action taken was better or worse than the average performance of the current policy. If the action was better (\textit{resp.}, worse) than the expectation, the probability of taking this action is increased (\textit{resp.}, decreased).
We refer the reader to \cite{schulman2017proximal} for the details of the PPO algorithm.

\subsection{Neural network architecture}

\begin{figure*}[t]
\centering
\includegraphics[width=\linewidth]{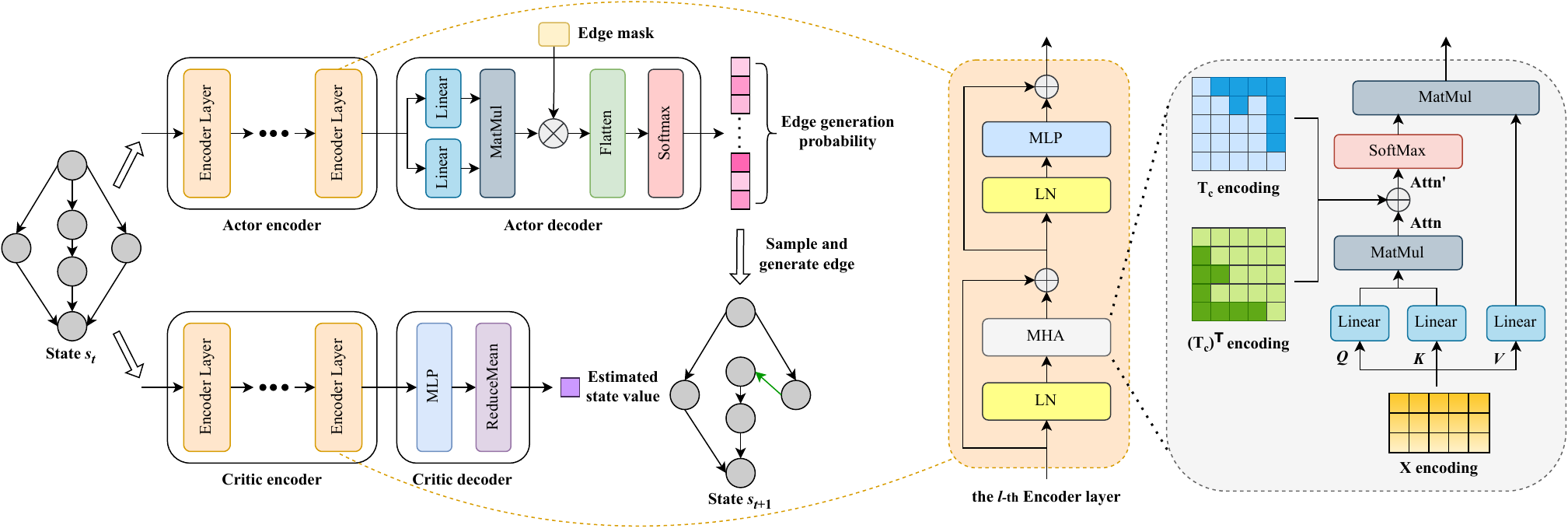}
\caption{Illustration of the actor and critic network.}
\label{fig:nn}
\end{figure*}

We apply an encoder-decoder architecture in the actor and critic network of the PPO algorithm. The encoder learns the node embedding of an input graph, and the decoder uses the node embedding to produce an output depending on its downstream task. Recall that in the PPO algorithm, the actor and critic have the same input but different outputs. Hence, we use the same encoder but different decoder architectures for the two networks. The overall architecture of the actor and critic networks is illustrated in \revise{Fig.}~\ref{fig:nn}. 

The encoder is built upon a recent graph representation network \textit{Graphormer} \cite{ying2021transformers}, which achieves state-of-the-art performance in various graph applications. Graphormer consists of multiple encoder layers, each of which includes a multi-head attention (MHA) module and a multi-layer perceptron (MLP) block with Layer normalization (LN) applied before the MHA and the MLP. The MHA is the key component in Graphormer, which effectively encodes the structural graph information via a residual term in the attention module. 
In this work, since a directed graph is considered, we encode both the forward and backward connectivity using the adjacency matrices of transitive closure $\matr{T}_c$ and its transpose $(\matr{T}_c)^\intercal$, respectively. Concretely, for each attention head, we assign different trainable scalars to each feasible value (\ie, $0$ and $1$) in $\matr{T}_c$ and $(\matr{T}_c)^\intercal$. Then, the trainable scalars corresponding to the $(i,j)$-th node pair will be added to the $(i,j)$-th entry of the attention product matrix $\matr{Attn} \in \mathbb{R}^{n\times n}$:
\begin{equation}\label{eq:attn}
    [\matr{Attn'}]_{ij} = [\matr{Attn}]_{ij} + b_1([\matr{T}_c]_{ij}) + b_2([\matr{T}_c^\intercal]_{ij}), \forall i,j
\end{equation}
where $b_1([\matr{T}_c]_{ij})$ and $b_2([\matr{T}_c^\intercal]_{ij})$\ denote the trainable scalars corresponding to $[\matr{T}_c]_{ij}$ and $[\matr{T}_c^\intercal]_{ij}$, respectively. 
Note that the trainable scalars are different in different attention heads, and shared across all encoder layers. For simplicity of presentation, we illustrate the single-head attention in \revise{Fig.}~\ref{fig:nn} and Equation~\eqref{eq:attn}. The extension to the multi-head attention is standard and straightforward.
\revise{We note that there are many other design choices of graph representation networks, \eg, graph neural networks (GNNs). A comprehensive survey can be found in \cite{wu2020comprehensive}. Evaluating different design choices of graph representation network architectures is beyond the scope of this paper.}

The decoder of the actor-network transforms the node embedding learned from the encoder network into edge generation probabilities. It first applies two linear layers to generate two linear transformations of each node embedding. Next, it conducts an inner product between the linear transformations of every two nodes to derive a scalar representing the edge generation score between the two nodes. Written in the matrix form, we have:
\begin{equation}
    \matr{S} = (\matr{H} \matr{W_1}) (\matr{H} \matr{W_2})^\intercal
\end{equation}
where $\matr{W_1} \in \mathbb{R}^{d \times d}$ and $\matr{W_2} \in \mathbb{R}^{d \times d}$ denote the linear transformation matrices; $\matr{H} \in \mathbb{R}^{n\times d}$ denotes the node embedding learned from the encoder network; $\matr{S} \in \mathbb{R}^{n \times n}$ denotes the edge generation score matrix where $[\matr{S}]_{ij}$ is the edge generation score from node $v_i$ to $v_j$.
Then, the action mask is used to mask out the ineligible edges by setting their generation score to $0$. Finally, a $\mathrm{Flatten}$ layer transforms the edge generation score matrix into a long vector, and a $\mathrm{Softmax}$ layer is applied to convert the edge generation scores into probabilities.

The decoder of the critic network is given by a MLP shared across different nodes. The input dimension of the MLP is the same as the dimension of node embedding, and the output dimension equals one. The MLP downscales each node's embedding into a scalar. The mean of scalars corresponding to all nodes is then used to estimate the cumulative return $\mathcal{R}(s)$ of the current input state $s$.

\section{Evaluation}
\label{sec:evaluation}

\subsection{DAG task generation}

We follow the DAG generation method proposed in \cite{melani2015response} to generate random DAG tasks for evaluation. The C++ implementation of the DAG generation method is available online\footnote{\url{https://github.com/mive93/DAG-scheduling}}.
To evaluate the proposed algorithm on DAG tasks with various characteristics, we generate a DAG task set with two varied parameters: (i) task utilization defined as the sum of sub-task utilization (\textit{i.e.}, $U = \sum_{i=1}^n{U_i}$, where $U_i = C_i / D$); (ii) task density defined as the ratio of the DAG length and its deadline (\textit{i.e.}, $dens = \mathrm{L}(\mathcal{G}) / D$). For the task utilization, we consider seven ranges $U \in [\underline{U}, \underline{U}+1)$ with $\underline{U}$ varied from $\{1.0, 2.0, ..., 7.0\}$. For the task density, we consider five ranges $dens \in [\underline{dens}, \underline{dens}+0.1)$ with $\underline{dens}$ varied from $\{0.5,0.6,...,0.9\}$. 
For each combination of the above parameter ranges, we generate $3,000$ random DAG tasks, which constitute a set of $7 \times 5 \times 3,000 = 105,000$ tasks. Then, we randomly split the whole task set into train, validation, and test sets using a ratio of $0.6:0.2:0.2$ \revise{, applying the split equally for each parameter range. We use the splits} respectively for training the DRL agent, tuning the DRL hyperparameters and the final evaluation of the pre-trained DRL agent and other comparison algorithms.

We note that the task generation method does not support specifying the number of nodes as a parameter of DAG generation. To show the variety of $n$ in our generated task set, we plot a histogram of $n$ in Fig.~\ref{fig:n_hist}. From the figure, we can see that $n\in[0, 140]$, with the majority of cases falling between $10$ and $80$.

\begin{figure}[ht]
\centering
\includegraphics[width=0.85\linewidth]{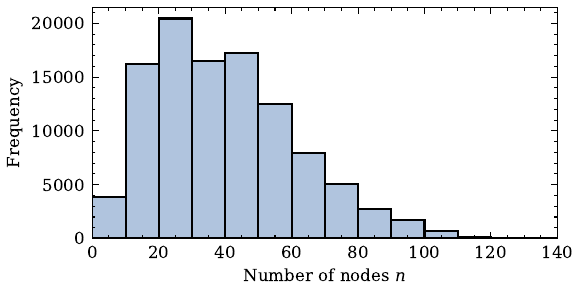}
\caption{Histogram of number of nodes $n$.}
\label{fig:n_hist}
\end{figure}

\begin{table*} [t]
\center
\caption{\revise{Comparison results on processor usage.}}\label{tab:res}
\newcommand{\colwidth}{1.5cm}
\begin{tabular}{ll|p{\colwidth}<{\raggedleft}p{\colwidth}<{\raggedleft}p{\colwidth}<{\raggedleft}p{\colwidth}<{\raggedleft}p{\colwidth}<{\raggedleft}p{\colwidth}<{\raggedleft}}
\hline
Utilization  & Density  & \texttt{EGS-PPO}  & \texttt{EGS-GRD}  &  \texttt{EGS-RND}  & \texttt{He2019}  & \texttt{Zhao2020}  & \texttt{He2021}  \\ \hline
& $[0.5, 0.6)$ & $\mathbf{2.06^{0.52}}$  &$2.22^{0.61}$& $2.26^{0.63}$  & $\mathbf{2.06^{0.51}}$  & $2.09^{0.53}$  &$\mathbf{2.06^{0.51}}$\\
             & $[0.6, 0.7)$ & $\mathbf{2.10^{0.54}}$  &$2.19^{0.59}$& $2.23^{0.62}$  & $2.14^{0.57}$  & $2.15^{0.57}$  &$2.12^{0.55}$\\
 $[1.0, 2.0)$ & $[0.7, 0.8)$ & $\mathbf{2.14^{0.56}}$  &$2.20^{0.60}$& $2.28^{0.64}$  & $2.20^{0.60}$  & $2.23^{0.61}$  &$2.17^{0.59}$\\
             & $[0.8, 0.9)$ & $\mathbf{2.20^{0.60}}$  &$2.29^{0.66}$& $2.31^{0.66}$  & $2.29^{0.66}$  & $2.33^{0.71}$  &$2.27^{0.65}$\\
             & $[0.9, 1.0)$ & $\mathbf{2.39^{0.75}}$  &$2.46^{0.81}$& $2.48^{0.79}$  & $2.47^{0.81}$  & $2.55^{0.91}$  &$2.47^{0.80}$\\\hline
 & $[0.5, 0.6)$ & $\mathbf{3.32^{0.47}}$  &$3.73^{0.57}$& $3.68^{0.56}$  & $3.38^{0.49}$  & $3.41^{0.49}$  &$3.36^{0.48}$\\
             & $[0.6, 0.7)$ & $\mathbf{3.38^{0.49}}$  &$3.74^{0.59}$& $3.78^{0.57}$  & $3.50^{0.51}$  & $3.56^{0.53}$  &$3.50^{0.50}$\\
$[2.0, 3.0)$             & $[0.7, 0.8)$ & $\mathbf{3.49^{0.51}}$  &$3.79^{0.62}$& $3.84^{0.65}$  & $3.65^{0.56}$  & $3.70^{0.59}$  &$3.60^{0.54}$\\
             & $[0.8, 0.9)$ & $\mathbf{3.67^{0.58}}$  &$3.99^{0.66}$& $4.05^{0.69}$  & $3.92^{0.67}$  & $4.09^{0.76}$  &$3.90^{0.63}$\\
             & $[0.9, 1.0)$ & $\mathbf{3.99^{0.74}}$  &$4.33^{0.83}$& $4.35^{0.84}$  & $4.32^{0.93}$  & $4.58^{1.03}$  &$4.28^{0.86}$\\\hline
 & $[0.5, 0.6)$ & $\mathbf{4.17^{0.38}}$  &$4.82^{0.57}$& $4.74^{0.52}$  & $4.39^{0.49}$  & $4.46^{0.50}$  &$4.37^{0.49}$\\
             & $[0.6, 0.7)$ & $\mathbf{4.44^{0.51}}$  &$4.97^{0.57}$& $4.90^{0.56}$  & $4.63^{0.54}$  & $4.72^{0.55}$  &$4.63^{0.50}$\\
$[3.0, 4.0)$             & $[0.7, 0.8)$ & $\mathbf{4.60^{0.56}}$  &$5.07^{0.62}$& $5.10^{0.63}$  & $4.90^{0.56}$  & $5.01^{0.64}$  &$4.85^{0.59}$\\
             & $[0.8, 0.9)$ & $\mathbf{4.91^{0.67}}$  &$5.47^{0.75}$& $5.51^{0.79}$  & $5.33^{0.82}$  & $5.59^{0.92}$  &$5.24^{0.72}$\\
             & $[0.9, 1.0)$ & $\mathbf{5.39^{0.84}}$  &$5.95^{0.97}$& $6.06^{0.98}$  & $5.95^{1.16}$  & $6.44^{1.43}$  &$5.87^{1.09}$\\\hline
 & $[0.5, 0.6)$ & $\mathbf{5.48^{0.51}}$  &$6.26^{0.68}$& $6.07^{0.61}$  & $5.59^{0.51}$  & $5.67^{0.56}$  &$5.60^{0.51}$\\
             & $[0.6, 0.7)$ & $\mathbf{5.63^{0.55}}$  &$6.49^{0.74}$& $6.35^{0.66}$  & $5.90^{0.60}$  & $5.97^{0.57}$  &$5.86^{0.56}$\\
$[4.0, 5.0)$             & $[0.7, 0.8)$ & $\mathbf{5.89^{0.60}}$  &$6.71^{0.85}$& $6.60^{0.79}$  & $6.24^{0.66}$  & $6.44^{0.81}$  &$6.20^{0.68}$\\
             & $[0.8, 0.9)$ & $\mathbf{6.30^{0.73}}$  &$7.12^{0.98}$& $7.19^{0.96}$  & $6.80^{0.92}$  & $7.20^{1.15}$  &$6.71^{0.85}$\\
             & $[0.9, 1.0)$ & $\mathbf{6.87^{0.99}}$  &$7.71^{1.13}$& $7.92^{1.25}$  & $7.77^{1.63}$  & $8.55^{1.99}$  &$7.47^{1.26}$\\\hline
 & $[0.5, 0.6)$ & $\mathbf{6.40^{0.49}}$  &$7.51^{0.70}$& $7.22^{0.62}$  & $6.67^{0.53}$  & $6.80^{0.54}$  &$6.63^{0.51}$\\
             & $[0.6, 0.7)$ & $\mathbf{6.68^{0.53}}$  &$7.82^{0.77}$& $7.56^{0.77}$  & $7.05^{0.64}$  & $7.19^{0.69}$  &$7.04^{0.62}$\\
$[5.0, 6.0)$             & $[0.7, 0.8)$ & $\mathbf{7.01^{0.59}}$  &$8.05^{0.79}$& $7.95^{0.75}$  & $7.53^{0.69}$  & $7.72^{0.79}$  &$7.39^{0.65}$\\
             & $[0.8, 0.9)$ & $\mathbf{7.47^{0.72}}$  &$8.61^{0.98}$& $8.60^{1.07}$  & $8.13^{1.03}$  & $8.54^{1.19}$  &$8.00^{0.88}$\\
             & $[0.9, 1.0)$ & $\mathbf{8.41^{1.05}}$  &$9.62^{1.28}$& $9.66^{1.36}$  & $9.71^{1.95}$  & $10.73^{2.47}$ &$9.14^{1.38}$\\\hline
 & $[0.5, 0.6)$ & $\mathbf{7.66^{0.54}}$  &$9.01^{0.82}$& $8.63^{0.74}$  & $7.88^{0.57}$  & $8.01^{0.60}$  &$7.85^{0.59}$\\
             & $[0.6, 0.7)$ & $\mathbf{8.03^{0.59}}$  &$9.37^{0.87}$& $9.05^{0.83}$  & $8.36^{0.72}$  & $8.55^{0.75}$  &$8.32^{0.69}$\\
$[6.0, 7.0)$             & $[0.7, 0.8)$ & $\mathbf{8.48^{0.80}}$  &$9.77^{1.05}$& $9.73^{1.05}$  & $9.05^{0.88}$  & $9.25^{0.98}$  &$8.91^{0.83}$\\
             & $[0.8, 0.9)$ & $\mathbf{9.03^{0.89}}$  &$10.41^{1.11}$& $10.45^{1.23}$  & $9.81^{1.11}$  & $10.19^{1.39}$ &$9.60^{0.96}$\\
             & $[0.9, 1.0)$ & $\mathbf{10.10^{1.42}}$ &$11.71^{1.74}$& $11.89^{1.71}$  & $11.83^{2.62}$ & $13.14^{3.34}$ &$11.20^{2.00}$\\\hline
 & $[0.5, 0.6)$ & $\mathbf{9.08^{0.65}}$  &$10.91^{1.06}$& $10.26^{0.86}$  & $9.28^{0.65}$  & $9.42^{0.68}$  &$9.26^{0.65}$\\
             & $[0.6, 0.7)$ & $\mathbf{9.42^{0.76}}$  &$11.21^{1.17}$& $10.73^{0.96}$  & $9.88^{0.81}$  & $10.03^{0.84}$ &$9.79^{0.80}$\\
$[7.0, 8.0)$             & $[0.7, 0.8)$ & $\mathbf{10.02^{0.90}}$ &$11.72^{1.26}$& $11.59^{1.15}$  & $10.69^{1.05}$ & $10.84^{1.05}$ &$10.55^{1.00}$\\
             & $[0.8, 0.9)$ & $\mathbf{10.78^{1.06}}$ &$12.60^{1.43}$& $12.73^{1.44}$  & $11.71^{1.39}$ & $12.28^{1.79}$ &$11.46^{1.21}$\\
             & $[0.9, 1.0)$ & $\mathbf{12.00^{1.80}}$ &$14.13^{2.20}$& $14.37^{2.35}$  & $14.32^{3.49}$ & $15.76^{4.22}$ &$13.10^{2.18}$\\\hline
\end{tabular}

\end{table*}

\subsection{Comparison algorithms}
We evaluate the proposed EGS framework and DRL algorithm by comparing the performance of the following DAG scheduling algorithms on the generated task set. 

\textbf{Edge generation scheduling heuristics}. 
\revise{
We consider three heuristics based on the proposed EGS framework, integrated with different edge generation policies, \ie, PPO policy (\texttt{EGS-PPO}), greedy policy (\texttt{EGS-GRD}), and random policy (\texttt{EGS-RND}). 
In each iteration of EGS,
\texttt{EGS-PPO} generates the edge with the maximum confidence given by a pre-trained PPO agent. 
\texttt{EGS-GRD} selects an edge that leads to the maximum intermediate width reduction, with a tie-breaking strategy of selecting the one with the minimum intermediate length increase.
\texttt{EGS-RND} generates an edge uniformly at random.
We note that all three policies only generate edges that are deemed eligible by the edge masks, as defined in the EGS framework.
The effectiveness of the proposed PPO policy can be evaluated through the comparison with \texttt{EGS-GRD} and \texttt{EGS-RND}.
}

\textbf{Mixed-integer linear programming} (\texttt{MILP}). We formulate the DAG scheduling problem as a mixed-integer linear program that can be solved by standard mathematical programming solvers to obtain optimal solutions. Although the optimality can be guaranteed, \texttt{MILP} is not computationally tractable. Thus, it can only be used to solve relatively small instances within a reasonable time (in our experiments, DAGs with $n \leq 20$ are solved by \texttt{MILP} within a 2-hour time limit). Through the comparison with \texttt{MILP}, the optimality gap of each comparison algorithm can be acquired. The detailed formulation of the \texttt{MILP} is reported in Appendix~\ref{sec:milp}.

\textbf{State-of-the-art DAG scheduling heuristics}. The proposed EGS framework is compared with three state-of-the-art DAG scheduling algorithms: \texttt{He2019} \cite{he2019intra}, \texttt{Zhao2020} \cite{zhao2020dag}\revise{, and \texttt{He2021} \cite{he2021response}}. These algorithms are all developed based on \textit{list scheduling} framework. The main differences are the priorities assigned to the DAG nodes. 
\revise{For example, \texttt{He2021} uses the vertex lengths (\ie, $t_i^{\text{EFT}} + D - t_i^{\text{LFT}}$) as node priorities},
while \texttt{He2019} and \texttt{Zhao2020} developed more sophisticated priority assignment rules.
We note that the objective used in these list scheduling heuristics is to minimize the \textit{makespan} of a DAG task given a fixed number of processors. However, they can be easily adapted to minimize processor usage given task deadline through an incremental search as shown in \revise{Appendix~\ref{sec:list_sched}}.

\subsection{Experimental setup}

All the experiments are conducted on a workstation equipped with AMD EPYC 7763 CPUs and Nvidia A100 GPUs running GNU/Linux. The proposed EGS framework is implemented in Python 3.8.12, and the PPO algorithm is implemented using Tensorflow 2.7.0. The \texttt{MILP} is solved by a standard mathematical programming solver Gurobi 9.5.0\footnote{\url{https://www.gurobi.com}} with a Python interface.
The hyper-parameters used in our PPO implementation are reported in \revise{Appendix~\ref{sec:ppo_hyper}}.

\revise{
The training of the DRL agent takes approximately $30$ hours. The trained actor network is used to make edge generation decisions on unseen tasks in the test set, taking on average $3.53$ seconds to schedule one DAG task.
}

\subsection{Comparison results}

\subsubsection{Overall results}

Table~\ref{tab:res} compares the proposed EGS algorithms with the state-of-the-art DAG scheduling heuristics. It summarizes each algorithm's average processor usage and associated standard deviation (indicated in the superscript) under different task utilization and density. We use \textbf{boldface} type to indicate the best results within the comparison. The results show that \texttt{EGS-PPO} algorithm outperforms other list scheduling heuristics in terms of processor usage across all tested utilization and density levels, which demonstrates the effectiveness of the proposed EGS framework and PPO policy. Additionally, the performance gain of \texttt{EGS-PPO} improves as the task utilization and density increase. 
\revise{\texttt{EGS-GRD} and \texttt{EGS-RND} show a similar trend. In particular, they perform worse than other algorithms at lower densities but perform better than \texttt{Zhao2020} at higher densities. This indicates that EGS framework has more potential to schedule DAGs with higher densities than list scheduling methods.

By comparing \texttt{EGS-GRD} and \texttt{EGS-RND}, it shows that the greedy policy performs better than the random when utilization is low, or density is high. This is expected since the greedy policy tends to lead the decision process to local optima, thus it is more difficult to achieve globally good performance for a long decision episode.
}

To further understand how the algorithms perform compared to the optimal solution, we compute their optimality gaps, which are defined as the relative performance deviations between the test algorithms and the optimal solution, \ie, $\frac{M_{alg} - M_{opt}}{M_{opt}}$,
where $M_{alg}$ and $M_{opt}$ denote the number of processors used by the test algorithm and the optimal solution, respectively. In our experiments, the optimal solutions are obtained by solving \texttt{MILP}. Since the \texttt{MILP} solver is not computationally tractable, we run it on DAGs with $n \leq 20$ and present the average optimality gap of each comparison algorithm in Table~\ref{tab:opt}. It shows that \texttt{EGS-PPO} achieves the best optimality gap (smaller than $2\%$) among all comparison algorithms. 
\revise{Moreover, \texttt{EGS-GRD} and \texttt{EGS-RND} achieve similar optimality gaps to \texttt{Zhao2020}.}

\begin{table} [t]
\center
\caption{\revise{Average optimality gap $\frac{M_{alg} - M_{opt}}{M_{opt}}$ when $n \leq 20$.}}
\label{tab:opt}

\begin{tabular}{p{1cm}p{1cm}p{1cm}p{1cm}p{1.2cm}p{1cm}}

\hline
\texttt{EGS-PPO} & \texttt{EGS-GRD} & \texttt{EGS-RND}  &   \texttt{He2019} & \texttt{Zhao2020} & \texttt{He2021}   \\
\hline
\textbf{1.78\%}   &7.20\%  & 8.92\%   & 5.95\%   & 8.86\%   & 5.29\%     \\ 
\hline

\end{tabular}

\end{table}

\subsubsection{Sensitivity of task utilization}

Fig.~\ref{fig:res_util} illustrates the average number of processors required by each algorithm under different task utilization $U$. The figure shows similar findings as in Table~\ref{tab:res} that \texttt{EGS-PPO} outperforms other algorithms for all task utilization levels, and the performance gain increases with task utilization. 
\revise{With regards to the list scheduling heuristics, \texttt{He2021} achieves the best performance, demonstrating that its priority assignment is better than \texttt{He2019} and \texttt{Zhao2020}.}

\revise{
Additionally, we compare the acceptance ratio (\ie, $\frac{\text{\# schedulable tasks}}{\text{\# tasks in the test set}}$, for each utilization level with increments of $0.1$) achieved by each algorithm given $M=8$. We can see from Fig.~\ref{fig:acc_util} that \texttt{EGS-PPO} achieves the best acceptance ratio among all algorithms. While \texttt{EGS-GRD} and \texttt{EGS-RND} achieve high acceptance ratio for low utilization tasks (\ie, better than \texttt{He2019} and \texttt{Zhao2020} for $U < 4$ and $U < 5$, respectively), their performance drops rapidly as $U > 6$. 
This is because the advantage of \texttt{EGS-GRD} and \texttt{EGS-RND} lies in scheduling high-density tasks (see Table~\ref{tab:res}). As $U \to M$, most high-density tasks require more than $M$ processors for all algorithms (see Table~\ref{tab:res}). Thus, the advantage of achieving lower processor usage for high-density tasks cannot contribute to the acceptance ratio. More details about the impact of task density on processor usage are discussed in the following.
}

\begin{figure}[t]
\centering
\includegraphics[width=0.85\linewidth]{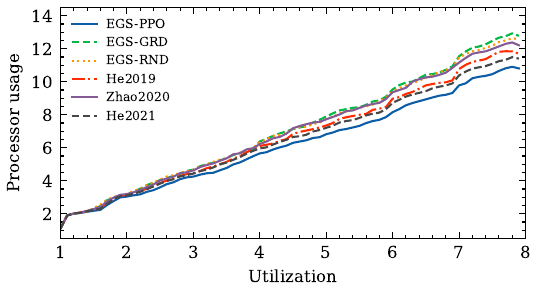}
\caption{\revise{Average processor usage with various $U$.}}
\label{fig:res_util}
\end{figure}


\begin{figure}[t]
\centering
\includegraphics[width=0.85\linewidth]{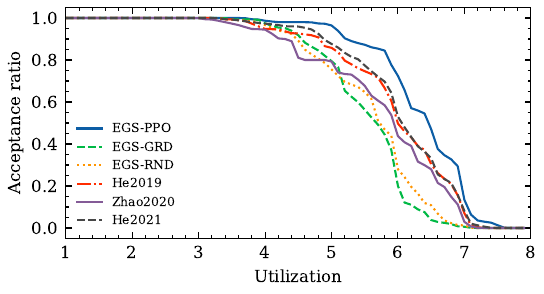}
\caption{\revise{Acceptance ratio with $M=8$ and various $U$.}}
\label{fig:acc_util}
\end{figure}

\subsubsection{Sensitivity of task density}

Fig.~\ref{fig:res_dens} illustrates the processor usage with different task densities with a violin plot. The ticks of each violin show the maximum, average, and minimum (from top to bottom) processor usage among all the test instances within each density range. The figure shows that the overall processor usage increases with task density, and \texttt{EGS-PPO} outperforms other algorithms in terms of average processor usage and performance stability, as indicated by the violin size. In particular, when $dens \geq 0.9$, \texttt{EGS-PPO} can save up to $5$ and $8$ processors compared to \texttt{He2019} and \texttt{Zhao2020}, respectively. Moreover, it shows similar findings as in Table~\ref{tab:res} that \revise{\texttt{EGS-GRD} and \texttt{EGS-RND} perform better than \texttt{He2019} and \texttt{Zhao2020} for high task density ($dens \geq 0.9$).}

\begin{figure}[ht]
\centering
\includegraphics[width=0.85\linewidth]{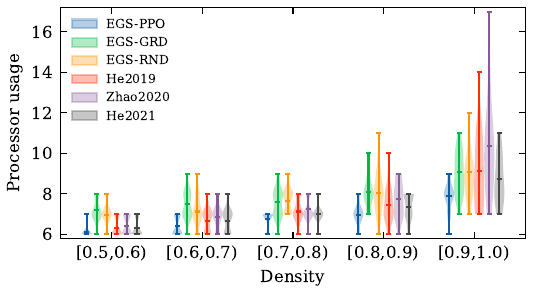}
\caption{\revise{Processor usage with $U=5$ and various $dens$.}}
\label{fig:res_dens}
\end{figure}

\section{Conclusion and Future Work}
\label{sec:conclusion}

In this paper, we studied the problem of scheduling real-time DAG tasks on multicore platforms to minimize processor usage while guaranteeing schedulability. We presented a new schedulability test based on the observation that a DAG task is schedulable if its width is not greater than the number of available processors and its length is less than or equal to the deadline. A new Edge Generation Scheduling (EGS) framework is proposed that converts a DAG task into a trivially schedulable DAG by iteratively adding edges. A DRL algorithm combined with a graph representation neural network is developed to learn an efficient edge generation policy for EGS. The effectiveness of three EGS variants (\ie, \texttt{EGS-PPO}, \texttt{EGS-GRD}, and \texttt{EGS-RND}) was evaluated by comparing to exact solutions and state-of-the-art DAG scheduling algorithms. Experimental results show that \texttt{EGS-PPO} outperforms other approaches, while \texttt{EGS-GRD} and \texttt{EGS-RND} achieve comparable results to the state-of-the-art for low-utilization and high-density tasks. 

Although the main focus of the paper is on the schedulability aspects of a DAG task, we note that the approach can also be extended to solve joint optimization of schedulability and other timing attributes such as reaction time or data age.
\revise{Additionally, given that the proposed method requires fewer processors than the state-of-the-art, it is expected to also provide better results when scheduling multiple tasks under a federated scheduling policy (\eg, \cite{li2014federated}).}

In the future, we plan to study real-time DAG scheduling in heterogeneous platforms. We are also interested in extending the proposed method to minimize the makespan of a DAG, which is important for production scheduling and cloud computing applications.

\bibliographystyle{IEEEtran}
\bibliography{IEEEabrv,drl_dag}

\begin{thebibliography}{10}
\providecommand{\url}[1]{#1}
\csname url@samestyle\endcsname
\providecommand{\newblock}{\relax}
\providecommand{\bibinfo}[2]{#2}
\providecommand{\BIBentrySTDinterwordspacing}{\spaceskip=0pt\relax}
\providecommand{\BIBentryALTinterwordstretchfactor}{4}
\providecommand{\BIBentryALTinterwordspacing}{\spaceskip=\fontdimen2\font plus
\BIBentryALTinterwordstretchfactor\fontdimen3\font minus \fontdimen4\font\relax}
\providecommand{\BIBforeignlanguage}[2]{{%
\expandafter\ifx\csname l@#1\endcsname\relax
\typeout{** WARNING: IEEEtran.bst: No hyphenation pattern has been}%
\typeout{** loaded for the language `#1'. Using the pattern for}%
\typeout{** the default language instead.}%
\else
\language=\csname l@#1\endcsname
\fi
#2}}
\providecommand{\BIBdecl}{\relax}
\BIBdecl

\bibitem{andreozzi_et_al:LIPIcs.ECRTS.2022.1}
M.~Andreozzi, G.~Gabrielli, B.~Venu, and G.~Travaglini, ``{Industrial Challenge 2022: A High-Performance Real-Time Case Study on Arm},'' in \emph{Euromicro Conference on Real-Time Systems (ECRTS)}, vol. 231, 2022, pp. 1:1--1:15.

\bibitem{waters-challenge-2019}
A.~Hamann, D.~Dasari, F.~Wurst, I.~Saudo, N.~Capodieci, P.~Burgio, and M.~Bertogna, ``{WATERS} industrial challenge,'' in \emph{Proceedings of the 10th International Workshop on Analysis Tools and Methodologies for Embedded Real-Time Systems (WATERS)}, 2019.

\bibitem{verucchi2020latency}
M.~Verucchi, M.~Theile, M.~Caccamo, and M.~Bertogna, ``Latency-aware generation of single-rate {DAG}s from multi-rate task sets,'' in \emph{IEEE Real-Time and Embedded Technology and Applications Symposium (RTAS)}, 2020, pp. 226--238.

\bibitem{ullman-np-complete-sched-prob}
J.~D. Ullman, ``{NP}-complete scheduling problems,'' \emph{Journal of Computer and System sciences}, vol.~10, no.~3, pp. 384--393, 1975.

\bibitem{micaela-thesis}
M.~Verucchi, ``A comprehensive analysis of {DAG} tasks: solutions for modern real-time embedded systems,'' Doctoral Dissertation, University of Modena and Reggio Emilia, Italy, 2020.

\bibitem{Li2022}
J.~Li, K.~Agrawal, and C.~Lu, ``Parallel real-time scheduling,'' in \emph{Handbook of Real-Time Computing}.\hskip 1em plus 0.5em minus 0.4em\relax Springer, 2022, pp. 447--467.

\bibitem{verucchi2023survey}
M.~Verucchi, I.~S. Olmedo, and M.~Bertogna, ``A survey on real-time {DAG} scheduling, revisiting the global-partitioned infinity war,'' \emph{Real-Time Systems}, vol.~59, no.~3, pp. 479--530, 2023.

\bibitem{baruah-sched-dag-assignment}
S.~Baruah, ``Scheduling {DAGs} when processor assignments are specified,'' in \emph{ACM International Conference on Real-Time Networks and Systems (RTNS)}, 2020, pp. 111--116.

\bibitem{chang-exact-dag-wcrt}
S.~Chang, J.~Sun, Z.~Hao, Q.~Deng, and N.~Guan, ``Computing exact {WCRT} for typed {DAG} tasks on heterogeneous multi-core processors,'' \emph{Journal of Systems Architecture}, vol. 124, p. 102385, 2022.

\bibitem{ahmed2022exact}
S.~Ahmed and J.~H. Anderson, ``Exact response-time bounds of periodic {DAG} tasks under server-based global scheduling,'' in \emph{IEEE Real-Time Systems Symposium (RTSS)}, 2022, pp. 447--459.

\bibitem{schulman2017proximal}
J.~Schulman, F.~Wolski, P.~Dhariwal, A.~Radford, and O.~Klimov, ``Proximal policy optimization algorithms,'' \emph{arXiv preprint arXiv:1707.06347}, 2017.

\bibitem{ying2021transformers}
C.~Ying, T.~Cai, S.~Luo, S.~Zheng, G.~Ke, D.~He, Y.~Shen, and T.-Y. Liu, ``Do transformers really perform badly for graph representation?'' \emph{Advances in Neural Information Processing Systems (NeurIPS)}, vol.~34, pp. 28\,877--28\,888, 2021.

\bibitem{Minaeva2021}
A.~Minaeva, D.~Roy, B.~Akesson, Z.~Hanzalek, and S.~Chakraborty, ``Control performance optimization for application integration on automotive architectures,'' \emph{{IEEE} Transactions on Computers}, vol.~70, no.~7, pp. 1059--1073, 2021.

\bibitem{AUTOSARTiming}
\BIBentryALTinterwordspacing
{AUTOSAR}, ``{Requirements on Timing Extensions},'' {Standard}, 2022. [Online]. Available: \url{https://www.autosar.org/fileadmin/standards/R22-11/FO/AUTOSAR_RS_TimingExtensions.pdf}
\BIBentrySTDinterwordspacing

\bibitem{ARINC653}
\BIBentryALTinterwordspacing
{ARINC}, ``{Avionics Application Software Standard Interface, Part 0, Overview of ARINC 653},'' {Standard}, 2021. [Online]. Available: \url{https://aviation-ia.sae-itc.com/events/avionics-application-executive-apex-software-subcommittee}
\BIBentrySTDinterwordspacing

\bibitem{Baruah12}
S.~Baruah, V.~Bonifaci, A.~{Marchetti-Spaccamela}, L.~Stougie, and A.~Wiese, ``A generalized parallel task model for recurrent real-time processes,'' in \emph{IEEE Real-Time Systems Symposium (RTSS)}, 2012, pp. 63--72.

\bibitem{Graham69}
R.~{L.~Graham}, ``Bounds on multiprocessing timing anomalies,'' \emph{SIAM Journal on Applied Mathematics}, vol.~17, no.~2, pp. 416--429, 1969.

\bibitem{he2019intra}
Q.~He, N.~Guan, Z.~Guo \emph{et~al.}, ``Intra-task priority assignment in real-time scheduling of {DAG} tasks on multi-cores,'' \emph{IEEE Transactions on Parallel and Distributed Systems}, vol.~30, no.~10, pp. 2283--2295, 2019.

\bibitem{zhao2020dag}
S.~Zhao, X.~Dai, I.~Bate, A.~Burns, and W.~Chang, ``{DAG} scheduling and analysis on multiprocessor systems: Exploitation of parallelism and dependency,'' in \emph{IEEE Real-Time Systems Symposium (RTSS)}, 2020, pp. 128--140.

\bibitem{he2021response}
Q.~He, M.~Lv, and N.~Guan, ``Response time bounds for {DAG} tasks with arbitrary intra-task priority assignment,'' in \emph{Euromicro Conference on Real-Time Systems (ECRTS)}, 2021, pp. 8:1--8:21.

\bibitem{he2022bounding}
Q.~He, N.~Guan, M.~Lv, X.~Jiang, and W.~Chang, ``Bounding the response time of {DAG} tasks using long paths,'' in \emph{IEEE Real-Time Systems Symposium (RTSS)}, 2022, pp. 474--486.

\bibitem{Bonifaci13}
V.~Bonifaci, A.~Marchetti-Spaccamela, S.~Stiller, and A.~Wiese, ``Feasibility analysis in the sporadic {DAG} task model,'' in \emph{Euromicro Conference on Real-Time Systems (ECRTS)}, 2013, pp. 225--233.

\bibitem{Baruah14}
S.~Baruah, ``Improved multiprocessor global schedulability analysis of sporadic {DAG} task systems,'' in \emph{Euromicro Conference on Real-Time Systems (ECRTS)}, 2014, pp. 97--105.

\bibitem{li2014federated}
J.~Li, J.~J. Chen, K.~Agrawal, C.~Lu, C.~Gill, and A.~Saifullah, ``Analysis of federated and global scheduling for parallel real-time tasks,'' in \emph{26th Euromicro Conference on Real-Time Systems (ECRTS)}, 2014, pp. 85--96.

\bibitem{Pathan18}
R.~Pathan, P.~Voudouris, and P.~Stenström, ``Scheduling parallel real-time recurrent tasks on multicore platforms,'' \emph{IEEE Transactions on Parallel and Distributed Systems}, vol.~29, no.~4, pp. 915--928, 2018.

\bibitem{yadlapalli2021lag}
Y.~Yadlapalli and C.~Liu, ``{LAG}-based analysis techniques for scheduling multiprocessor hard real-time sporadic {DAGs},'' in \emph{IEEE Real-Time Systems Symposium (RTSS)}, 2021, pp. 316--328.

\bibitem{zhao2022dag}
S.~Zhao, X.~Dai, and I.~Bate, ``{DAG} scheduling and analysis on multi-core systems by modelling parallelism and dependency,'' \emph{IEEE Transactions on Parallel and Distributed Systems}, vol.~33, no.~12, pp. 4019--4038, 2022.

\bibitem{Baruah15}
S.~Baruah, V.~Bonifaci, and A.~Marchetti-Spaccamela, ``The global edf scheduling of systems of conditional sporadic {DAG} tasks,'' in \emph{Euromicro Conference on Real-Time Systems (ECRTS)}, 2015, pp. 222--231.

\bibitem{melani2015response}
A.~Melani, M.~Bertogna, V.~Bonifaci, A.~Marchetti-Spaccamela, and G.~C. Buttazzo, ``Response-time analysis of conditional {DAG} tasks in multiprocessor systems,'' in \emph{Euromicro Conference on Real-Time Systems (ECRTS)}, 2015, pp. 211--221.

\bibitem{ueter2021response}
N.~Ueter, M.~G{\"u}nzel, and J.-J. Chen, ``Response-time analysis and optimization for probabilistic conditional parallel {DAG} tasks,'' in \emph{IEEE Real-Time Systems Symposium (RTSS)}, 2021, pp. 380--392.

\bibitem{Yang16}
K.~Yang, M.~Yang, and J.~{H.~Anderson}, ``Reducing response-time bounds for {DAG}-based task systems on heterogeneous multicore platforms,'' in \emph{ACM International Conference on Real-Time Networks and Systems (RTNS)}, 2016, pp. 349--358.

\bibitem{Chang20}
S.~Chang, X.~Zhao, Z.~Liu, and Q.~Deng, ``Real-time scheduling and analysis of parallel tasks on heterogeneous multi-cores,'' \emph{Journal of Systems Architecture}, vol. 105, p. 101704, 2020.

\bibitem{Zahaf21}
H.~Zahaf, N.~Capodieci, R.~Cavicchioli, G.~Lipari, and M.~Bertogna, ``The {HPC-DAG} task model for heterogeneous real-time systems,'' \emph{IEEE Transactions on Computers}, vol.~70, no.~10, pp. 1747--1761, 2021.

\bibitem{reghenzani2021multi}
F.~Reghenzani, A.~Bhuiyan, W.~Fornaciari, and Z.~Guo, ``A multi-level {DPM} approach for real-time {DAG} tasks in heterogeneous processors,'' in \emph{IEEE Real-Time Systems Symposium (RTSS)}, 2021, pp. 14--26.

\bibitem{bi2022response}
R.~Bi, Q.~He, J.~Sun, Z.~Sun, Z.~Guo, N.~Guan, and G.~Tan, ``Response time analysis for prioritized {DAG} task with mutually exclusive vertices,'' in \emph{IEEE Real-Time Systems Symposium (RTSS)}, 2022, pp. 460--473.

\bibitem{mao2019learning}
H.~Mao, M.~Schwarzkopf, S.~B. Venkatakrishnan, Z.~Meng, and M.~Alizadeh, ``Learning scheduling algorithms for data processing clusters,'' in \emph{Proceedings of the ACM special interest group on data communication (SIGCOMM)}, 2019, pp. 270--288.

\bibitem{sun2021deepweave}
P.~Sun, Z.~Guo, J.~Wang, J.~Li, J.~Lan, and Y.~Hu, ``Deepweave: Accelerating job completion time with deep reinforcement learning-based coflow scheduling,'' in \emph{International Conference on International Joint Conferences on Artificial Intelligence (IJCAI)}, 2021, pp. 3314--3320.

\bibitem{Lee2021Global}
H.~Lee, S.~Cho, Y.~Jang, J.~Lee, and H.~Woo, ``A global {DAG} task scheduler using deep reinforcement learning and graph convolution network,'' \emph{IEEE Access}, vol.~9, pp. 158\,548--158\,561, 2021.

\bibitem{jeon2023neural}
W.~Jeon, M.~Gagrani, B.~Bartan, W.~W. Zeng, H.~Teague, P.~Zappi, and C.~Lott, ``Neural {DAG} scheduling via one-shot priority sampling,'' in \emph{International Conference on Learning Representations (ICLR)}, 2023.

\bibitem{floyd1962algorithm}
R.~W. Floyd, ``Algorithm 97: shortest path,'' \emph{Communications of the ACM}, vol.~5, no.~6, p. 345, 1962.

\bibitem{dilworth1950decomposition}
R.~Dilworth, ``A decomposition theorem for partially ordered sets,'' \emph{Annals of Mathematics}, pp. 161--166, 1950.

\bibitem{hopcroft1973n}
J.~E. Hopcroft and R.~M. Karp, ``An $n^{5/2}$ algorithm for maximum matchings in bipartite graphs,'' \emph{SIAM Journal on computing}, vol.~2, no.~4, pp. 225--231, 1973.

\bibitem{bang2008digraphs}
J.~Bang-Jensen and G.~Z. Gutin, \emph{Digraphs: theory, algorithms and applications}.\hskip 1em plus 0.5em minus 0.4em\relax Springer Science \& Business Media, 2008.

\bibitem{schulman2015high}
J.~Schulman, P.~Moritz, S.~Levine, M.~Jordan, and P.~Abbeel, ``High-dimensional continuous control using generalized advantage estimation,'' \emph{arXiv preprint arXiv:1506.02438}, 2015.

\bibitem{wu2020comprehensive}
Z.~Wu, S.~Pan, F.~Chen, G.~Long, C.~Zhang, and S.~Y. Philip, ``A comprehensive survey on graph neural networks,'' \emph{IEEE transactions on neural networks and learning systems}, vol.~32, no.~1, pp. 4--24, 2020.

\end{thebibliography}

\appendices

\section{MILP Formulation}
\label{sec:milp}

The DAG scheduling problem is formulated as a mixed-integer linear program (MILP) with the notations in Table~\ref{tab:notation}.

\begin{table} [h]
\center
\begin{threeparttable}
\caption{Table of notations used in the MILP.}\label{tab:notation}

\setlength{\tabcolsep}{\mytabcolsepsmall}
\begin{tabular}{ll}
\hline
Notation & Implication \\

\hline

Problem data &  \\
$\mathcal{G}=(\mathcal{V},\mathcal{E})$ & Task graph with $n$ nodes (sub-tasks); \\

$D$ & Task deadline; \\

$m$ & Width of the task graph (\textit{i.e.}, $\mathrm{W}(\mathcal{G}) = m$); \\

$i,j$ & Node index, $i=1,...,n$, $j=i,...,n$; \\

$k$ & Processor index, $k=1,...,m$; \\

$C_i$ & WCET of node $i$; \\


$M_1, M_2$ & Two large constant numbers. \\

\hline

Decisions & \\
$x_{ik}$ & $1$, if node $v_i$ executes on processor $k$; $0$, otherwise; \\

$y_{k}$ & $1$, if any node executes on processor $k$; 0, otherwise;\\

$\gamma_{ij}$ & $1$, if node $v_i$ and $j$ execute on the same processor,\\
& and node $v_i$ executes later than $v_j$; $0$, otherwise; \\

$f_i$ & Finishing time of node $v_i$. \\

\hline

\end{tabular}

\end{threeparttable}

\end{table}
\begin{equation}\label{eq:obj}
    \minimize \quad \sum_{k=1}^{m}{y_k}
\end{equation}
subject to:
\begin{equation}\label{eq:sole_proc}
    \sum_{k=1}^{m}{x_{ik}} = 1, \quad \forall v_i \in \mathcal{V}
\end{equation}
\begin{multline}\label{eq:exe_order_1}
    f_{i} \leq f_{j} - C_{j} + M_1 \cdot \gamma_{ij}\\
    + M_2 \cdot (2 - x_{ik} - x_{jk}), \forall v_i \neq v_j, k=1,...,m
\end{multline}
\begin{multline}\label{eq:exe_order_2}
    f_{j} \leq f_{i} - C_{i} + M_1 \cdot (1 - \gamma_{ij})\\ + M_2 \cdot (2 - x_{ik} - x_{jk}), \forall v_i \neq v_j, k=1,...,m
\end{multline}
\begin{equation}\label{eq:start_time}
    C_i \leq f_i, \quad \forall v_i \in \mathcal{V}
\end{equation}
\begin{equation}\label{eq:finish_time}
    f_i \leq D, \quad \forall v_i \in \mathcal{V}
\end{equation}
\begin{equation}\label{eq:precedence_constraint}
    f_{i} + C_{j} \leq f_{j} , \quad e_{ij} \in \mathcal{E}
\end{equation}
\begin{equation}\label{eq:busy_proc}
    x_{ik} \leq y_k, \quad \forall v_i \in \mathcal{V}, k = 1,...,m
\end{equation}
Objective~\eqref{eq:obj} minimizes the number of processors used to schedule the DAG task. 
Constraints~\eqref{eq:sole_proc} ensure that each node is assigned to one and only one processor to execute. 
Constraints~\eqref{eq:exe_order_1} and \eqref{eq:exe_order_2} guarantee the execution order of the nodes assigned to the same processor and make sure there is at most one node running on each processor at each time instant. 
Constraints~\eqref{eq:start_time} and \eqref{eq:finish_time} ensure all the nodes start and finish their execution no earlier than the release time $0$ and no later than the deadline $D$, respectively. Constraints~\eqref{eq:precedence_constraint} implement the precedence constraints between the nodes. Constraints~\eqref{eq:busy_proc} indicate busy processors (\textit{i.e.}, the processors to which at least one node is assigned).

\revise{
\section{Incremental search for List Scheduling}
\label{sec:list_sched}

Algorithm~\ref{alg:list_sched} presents the incremental search procedures of minimizing processor usage with list scheduling heuristics.

{
\SetAlCapFnt{\small}
\SetAlCapNameFnt{\small}
\SetAlFnt{\small}
\begin{algorithm} [h]
    \caption{Incremental search to minimize processor usage with list scheduling heuristics}
    \label{alg:list_sched}

    \KwIn{($\mathcal{G},D)$: the DAG task to be scheduled\;}
    \KwOut{$M^*$: number of processors used\;}
    \For{$M^* \leftarrow 1$ \KwTo $n$}
    {
        $Makespan \leftarrow \texttt{ListSched}(\mathcal{G}, M^*)$ \cite{he2019intra, zhao2020dag, verucchi2020latency}\;
        \If{$\mathcal{G}$ is schedulable (\textit{i.e.}, $Makespan \leq D$)}
        {
            \Return $M^*$\;
        }
    }
\end{algorithm}
}

\vspace{-1em}

\section{PPO Hyperparameters}
\label{sec:ppo_hyper}
We summarized the hyperparameters of the PPO implementation in Table~\ref{tab:hyperparameter}.

\begin{table} [ht]
\center
\caption{PPO hyper-parameters.}\label{tab:hyperparameter}

\setlength{\tabcolsep}{\mytabcolsepsmall}
\begin{tabular}{ll|ll}
\hline
Hyper-parameter & Value & Hyper-parameter & Value \\
\hline

Discount factor $\gamma$ & $0.99$ & Number of encoder layers & $2$\\

GAE parameter $\lambda$ & $0.97$ & Number of attention heads & $8$ \\

Clipping parameter $\epsilon$ & $0.2$ & Node embedding dimension & $64$ \\

Number of iterations & $500$ & MLP hidden dimension & $64$ \\

Length of rollout & $50,000$ &  Initial learning rate & $10^{-4}$\\


Batch size & $100$ & End learning rate & $10^{-5}$ \\

Epochs per iteration & $10$ & Learning rate decay steps & $10^6$ \\

\hline

\end{tabular}


\end{table}
}

\end{document}